%% file: main.tex
\documentclass[journal]{IEEEtran}

\ifCLASSOPTIONcompsoc
  \usepackage[nocompress]{cite}
\else
  \usepackage{cite}
\fi

\ifCLASSOPTIONcompsoc
  \usepackage[caption=false,font=normalsize,labelfont=sf,textfont=sf]{subfig}
\else
  \usepackage[caption=false,font=footnotesize]{subfig}
\fi

\input{macro.tex}

\begin{document}

\title{{\catro}: Channel Pruning via Class-Aware\\Trace Ratio Optimization}

\author{Wenzheng~Hu,
        Zhengping~Che,~\IEEEmembership{Member,~IEEE},
        Ning~Liu,~\IEEEmembership{Member,~IEEE},
        Mingyang~Li,\\
        Jian~Tang,~\IEEEmembership{Fellow,~IEEE},
        Changshui~Zhang,~\IEEEmembership{Fellow,~IEEE},
        and~Jianqiang~Wang
\thanks{W. Hu is with
  The State Key Laboratory of Automotive Satefy and Energy, Tsinghua University, Beijing, China, 100084
  and with Kuaishou Technology Co., Ltd., Beijing, China, 100085.
  (E-mail: hwz13@tsinghua.org.cn and huwenzheng@kuaishou.com)
}
\thanks{Z. Che, N. Liu, and J. Tang are with
  Midea Group, Beijing, China, 100102.
  (E-mail: chezp@midea.com, liuning22@midea.com, and tangjian22@midea.com)
}
\thanks{M. Li and C. Zhang are with
  Institute for Artificial Intelligence, Tsinghua University (THUAI),
  The State Key Lab of Intelligent Technologies and Systems,
  Beijing National Research Center for Information Science and Technology (BNRist),
  and Department of Automation, Tsinghua University, Beijing, China, 100084.
  (E-mail: gogolimingyang@gmail.com and zcs@mail.tsinghua.edu.cn)
}
\thanks{J. Wang is with
  The State Key Laboratory of Automotive Satefy and Energy, Tsinghua University, Beijing, China, 100084.
  (E-mail: wjqlws@tsinghua.edu.cn)
}
}

\markboth{{\catro}: Channel Pruning via Class-Aware Trace Ratio Optimization}%
{}
%

\maketitle

\begin{abstract}
  \input{abstract.tex}
\end{abstract}


%
\IEEEpeerreviewmaketitle

\section{Introduction}
\input{intro.tex}

\section{Related Work}

\input{related.tex}

\section{Methodology}
\input{method.tex}

\section{Experiments}
\input{exp.tex}

\section{Conclusion}
\input{conclusion.tex}
\section*{Acknowledgment}
This work was supported in part by the China Postdoctoral Science Foundation (No. 2020M680563),
and in part by NSFC under Grant No. 52131201.

\ifCLASSOPTIONcaptionsoff
  \newpage
\fi


\bibliographystyle{IEEEtran}
\bibliography{ref}

\input{bio.tex}

\end{document}

%% file: macro.tex
\usepackage{hyperref}       
\usepackage{url}            
\usepackage{amsfonts}       
\usepackage{nicefrac}       
\usepackage{microtype}      
\usepackage{xcolor}         
\usepackage{graphicx}
\usepackage{booktabs}                                   
\usepackage{multirow}                                   
\usepackage{algorithmic,algorithm}
\usepackage{tablefootnote}                              
\usepackage[symbol]{footmisc}                           
\usepackage{amsmath}                                    
\usepackage{amssymb}                                    
\usepackage{xcolor}                                     
\usepackage{enumitem}                                   

\newcommand{\field}[1]{\mathbb{#1}}
\newcommand{\vct}[1]{\mathbf{#1}}                   
\newcommand{\mat}[1]{\mathbf{#1}}                   
\newcommand{\set}[1]{\mathcal{#1}}                      
\newcommand{\func}[1]{{#1}}                             
\newcommand{\T}{^{\top}}                                
\newcommand{\R}{\field{R}}                              

\def\fF{\func{F}}
\def\mTheta{\mat{\Theta}}

\def\vx{\vct{x}}
\def\vy{\mathrm{y}}
\def\mX{\mat{X}}
\def\mY{\mat{y}}

\def\vc{\boldsymbol{c}}
\def\vd{\boldsymbol{d}}

\def\vo{\vct{o}}
\def\voo{\vct{\tilde{o}}}
\def\mO{\mat{O}}
\def\mOO{\mat{\tilde{O}}}
\def\vm{\vct{m}}
\def\mV{\mat{V}}
\def\sI{\set{I}}
\def\sII{\set{\widehat{I}}}
\def\sU{\set{U}}

\def\mVIl{\mV_{\sI_l}}

\def\mVIlopt{\mV_{\sI_{l}^{*}}}

\def\vel{\vct{e}^{(l)}}

\newcommand{\ii}[2]{{#1}^{[#2]}}

\def\Gb{\mat{G}^{(b)}}
\def\Gw{\mat{G}^{(w)}}
\def\Db{\mat{\Lambda}^{(b)}}
\def\Dw{\mat{\Lambda}^{(w)}}
\def\Gbij{\mathrm{G}^{(b)}_{i,j}}
\def\Gwij{\mathrm{G}^{(w)}_{i,j}}
\def\Dbii{\mathrm{\Lambda}^{(b)}_{i,i}}
\def\Dwii{\mathrm{\Lambda}^{(w)}_{i,i}}
\def\Gwl{\mat{G}^{(w)}_{l}}
\def\Gbl{\mat{G}^{(b)}_{l}}

\def\sS{\set{S}}
\def\fA{\mathcal{A}}
\def\fL{\mathcal{L}}
\def\fC{\mathcal{C}}

\def\fT{\mathcal{T}}
\def\fD{\mathcal{D}}
\def\pfT{\partial\fT}
\def\pfD{\partial\fD}
\def\fDT{\Gamma}

\def\nN{\mathcal{N}}
\def\nTB{\mathcal{T_B}}

\def\nsli{s_{l,i}}
\def\fHl{H_l}

\newtheorem{theorem}{\textbf{Theorem}}
\newtheorem{definition}{\textbf{Definition}}
\newtheorem{assumption}{\textbf{Assumption}}
\newtheorem{lemma}{\textbf{Lemma}}
\newenvironment{proof}{{\noindent\it \textbf{Proof}}\quad}{$\hfill\blacksquare$}

\newcommand{\eat}[1]{}                                  

\def\catro{{CATRO}}

%% file: abstract.tex
Deep convolutional neural networks are shown to be overkill with high parametric and computational redundancy in many application scenarios, and an increasing number of works have explored model pruning to obtain lightweight and efficient networks.
However, most existing pruning approaches are driven by empirical heuristic and rarely consider the joint impact of channels, leading to unguaranteed and suboptimal performance.
In this paper, we propose a novel channel pruning method via \textbf{c}lass-\textbf{a}ware \textbf{t}race \textbf{r}atio \textbf{o}ptimization (\catro) to reduce the computational burden and accelerate the model inference.
Utilizing class information from a few samples, {\catro} measures the joint impact of multiple channels by feature space discriminations and consolidates the layer-wise impact of preserved channels.
By formulating channel pruning as a submodular set function maximization problem,
{\catro} solves it efficiently via a two-stage greedy iterative optimization procedure.
More importantly, we present theoretical justifications on convergence of {\catro} and performance of pruned networks.
Experimental results demonstrate that {\catro} achieves higher accuracy with similar computation cost or lower computation cost with similar accuracy than other state-of-the-art channel pruning algorithms.
In addition, because of its class-aware property, {\catro} is suitable to prune efficient networks adaptively for various classification subtasks,
enhancing handy deployment and usage of deep networks in real-world applications.

%% file: intro.tex
\IEEEPARstart{I}{n} the past few years, deep convolution neural networks~(CNNs) have achieved impressive performance in computer vision tasks, especially image classification~\cite{krizhevsky2012imagenet,kaiming2016deep}.
However, deep models often have an enormous number of parameters, which requires colossal memory and massive amounts of computations.
These requirements not only increase infrastructure costs but also impose a great challenge to deploy models in resource-constrained environments, including mobile devices, embedded systems, and autonomous robots.
Therefore, it is significant to obtain deep models with high accuracy but relatively low computations used in various scenarios.
Pruning is an effective way to accelerate and compress deep networks by removing less important connections in the network.
Channel pruning~\cite{li2016pruning,he2019filter,kang2020operation}, as a hardware-friendly method, directly reduces redundancy computation by removing channels in convolutional layers and is widely used in practice.

Among recent advances in channel pruning, most approaches are driven by empirical heuristic and rarely consider the channel dependency.
Many pruning methods directly measure the importance of individual channels by the weights of filters~\cite{li2016pruning,he2018soft,he2019filter},
and others introduce some intuitive losses~\cite{luo2017thinet,zhuang2018discrimination,kang2020operation} about reconstruction, discrimination, sparsity, etc.
However, neglecting the dependency among channels leads to suboptimal pruning results.
For example, even multiple channels are important when examined independently, keeping them all in a pruned network may still be redundancy because of similar extracted features from different channels.
On the contrary, two or more individually unimportant channels may have boosted impact when working together.
Such joint impact have been moderately studied in machine learning methods~\cite{krishnapuram2004bayesian,nie2008trace,wang2015joint} but is rarely exploited in channel pruning.
CCP~\cite{peng2019collaborative} moves one step ahead with only pairwise channel correlations in pruning, and it remains demanding and challenging for effective utilizations and rigorous investigations of multiple channel joint impact.

\begin{figure}[t]
    \centering
    \includegraphics[width=0.9\columnwidth]{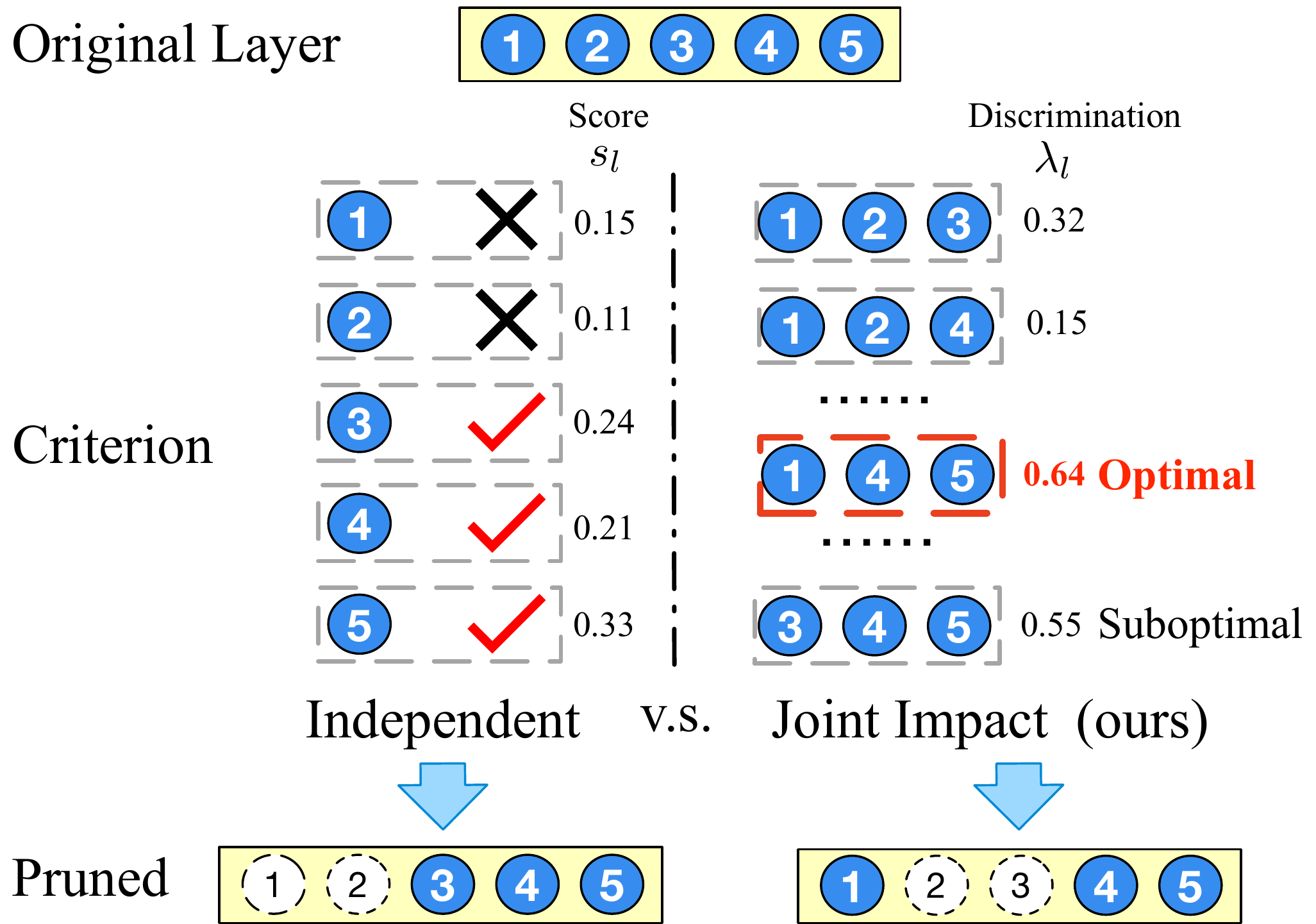}
    \caption{Illustrative comparisons of existing pruning methods treating channels independently with suboptimal results and our method ({\catro}) considering the joint impact of channels with optimal solutions.}
    \label{fig:vs}
\end{figure}

In this paper, we develop a novel channel pruning method. As shown in Fig.~\ref{fig:vs}, our method benefits from taking joint impact of channels into account and is superior to most existing methods pruning channels independently.
The proposed method extends network discriminative power~\cite{zhuang2018discrimination,tian2021task} for multi-channel setting and find channels by solving a combinatorial optimization problem.
To be more specific, we firstly transfer the channel combination optimization into a graph trace optimization,
and then we solve the trace optimization in a global view to obtain a result with the consideration of the joint impact.
By formulating the pruning objective into a submodular set function maximization problem,
we further show a theoretical lower boundary guarantee on accuracy and convergence under an assumption about the relationship between discrimination and accuracy.

In addition, different from most approaches measuring the channel importance by the filter weights, the discrimination in {\catro} is based on input samples and class-aware in the optimization steps.
This enables {\catro} to adaptively mine subsets of data samples and handle various pruning tasks with the same original network.
For example, suppose we have a large traffic sign classification model that works perfectly in the whole world, and we need to deploy a tiny model on an autonomous vehicle operated in a small town, where only a subset types of signs are on the road.
This kind customization and compression combination tasks are ubiquitous in the real world and require efficient solutions.
We name such application scenarios \emph{classification subtask compression} and showcase that {\catro} is quite suitable in these scenarios.

To the best of our knowledge,
{\catro} is the first to exploit the multi-channel joint impact to find the optimal combination of preserved channels in pruning
and the first to use submodularity for effective and theoretically guaranteed pruning performance.
We summarize our contributions as follows:
\begin{itemize}[itemsep=1pt,topsep=1pt,parsep=1pt,leftmargin=15pt] %
    \item We introduce graph-based trace ratio as discrimination criteria to consider the multi-channel joint impact for channel pruning.
    \item We formulate the pruning to a submodular maximization problem, provide theoretical guarantees on convergence of the pruning step and performance of pruned networks, and solve it with an efficient two-stage algorithm.
    \item We demonstrate in extensive experiments that our method achieves superior performance compared to existing channel pruning methods by the efficiency and accuracy of the pruned networks.
\end{itemize}

%% file: related.tex
\subsection{Channel Pruning}
The recent advances of weight pruning can be mainly categorized into \emph{unstructured pruning} and \emph{channel pruning}. \emph{Unstructured pruning} can effectively prune the over-parameterized deep neural networks (DNNs) but results in irregular sparse matrices, which require customized hardwares and libraries to support the practical speedup, leading to a limited inference acceleration in most practical cases~\cite{han2015,han2016deep_compression}. Recent work concentrates on channel pruning~\cite{liu2017learning,he2019filter,dai2019nest,lin2020hrank,he2020learning,DBLP:journals/corr/abs-2011-02166,zheng2022model,he2022filter,chen2021dynamical,tang2022automatic,peng2021overcoming}, which prunes the whole channel and results in structured sparsity. The compacted DNN after channel pruning can be directly deployed on the existing prevalent CNN acceleration frameworks without dedicated hardware/libraries implementation.
Early work~\cite{wen2016learning} proposes a straightforward pruning heuristic, penalizing weights with $l_{1}$-norm group lasso regularization and then eliminate the least $l_{1}$-norm of channels.
Network Slimming~\cite{liu2017learning} introduces a factor in the batch normalization layer with respect to each channel and evaluates the importance of channels by the magnitude of the factor. Then, many criteria are proposed to evaluate the importance of neuron.
SCP~\cite{he2018soft} measures the relative importance of a filter in each layer by calculating the $l_{1}$-norm and $l_{2}$-norm, respectively. FPGM~\cite{he2019filter} removes the filters minimizing the sum of distances to others.
Variational CNN pruning~\cite{zhao2019variational} provides a Bayesian model compression technique, which approximates batch-norm scaling parameters to a factorized normal distribution using stochastic variational inference. GBN~\cite{you2019gate} introduces a gate for each channel and performs a backward method to optimize the gates. Then it independently prunes channels by the rank of gates.
Similarily, DCPH~\cite{chen2021dynamical} designs a channelwise gate to enable or disable each channel.
Greedy algorithm is also applied to identify the subnetwork inside the original network~\cite{ye2020good}.
Recently, MFP~\cite{he2022filter} introduces a new set of criteria to consider the geometric distance of filters.
These works evaluate and prune the importance of channels individually.
Another work~\cite{tian2021task} leverage Fisher's linear discriminant analysis (LDA) to prune the last layer of the DNN, which, however is empirical heuristic and suffers from the expensive cost of cross-layer tracing.
CCP~\cite{peng2019collaborative} exploits the inter-channel dependency but stops with only pairwise correlations. It prunes channels without global consideration either.
SCOP~\cite{Tang2020SCOP} utilize samples to set up a scientific control and then prunes the neural network in the group by introducing generated input features. Recent prevailing AutoML-based pruning approaches~\cite{he2018amc,dong2019network,zechun2019Meta,lin2020channel,Liu2020Autocompress,LiWSW2020EagleEye} can achieve a high FLOPs reduction of DNNs.
For instance, AMC~\cite{he2018amc} leverages deep reinforcement learning (DRL) to let the agent optimizes the policy from extensive experiences by repeatedly pruning and evaluating for DNNs.
DNCP~\cite{zheng2022model} uses a set of learnable architecture parameters, and calculates the probability
of each channel being retained based on these parameters.
However, they usually requires tremendous computation budgets due to the large search space and trial-and-error manner.

\subsection{Submodular}
In recent years, the combinatorial selection based on submodular functions has been one of the most promising methods in machine learning and data mining. It has been a surge of interest in lots of computer tasks, such as visual recognition~\cite{zheng2016submodular}, segmentation~\cite{zhang2018exploring}, clustering~\cite{liu2013entropy,shen2019submodular}, active learning~\cite{DBLP:conf/icml/WeiIB15} and user recommendation~\cite{DBLP:conf/ijcai/AshkanKBW15}.
The submodular set function maximization~\cite{nemhauser1978analysis}, which is to maximize a submodular function, is one of the most basic and important problem. There exist several algorithms for solving nonnegative and monotone submodular function maximization. Under the uniform matroid constraint~\cite{nemhauser1978analysis}, a standard greedy algorithm gives
an $1-e^{-1}$
approximation~\cite{feige1998threshold,calinescu2007maximizing,chekuri2015multiplicative} in polynomial time, and
$1-e^{-1}$ is shown to be the best possible approximation unless $P = NP$~\cite{feige1998threshold}.
There are also many works investigating more effective algorithms~\cite{badanidiyuru2014streaming,buschjager2021very} and many other works investigating the maximization of a submodular set function under different constraints~\cite{iyer2013submodular,buchbinder2014submodular,sadeghi2020online}. Although submodularity has been widely used in many tasks and widely studied, it has not been explored in model compression as far as we known.

%% file: method.tex
\subsection{Problem Reformulation}
\input{method-formulation.tex}

\subsection{Trace Ratio Criterion for Channel Pruning}
\input{method-trace.tex}

\subsection{Iterative Layer-wise Optimization}
\input{method-creterion.tex}

\subsection{Greedy Architecture Optimization}
\label{sec:method-greedy}
\input{method-framework.tex}

\subsection{Overall Pruning Procedure}
\input{method-all.tex}

%% file: method-formulation.tex
For a CNN $\fF(\cdot; \mTheta)$ with a stack of $L$ convolution (CONV) layers and weight parameters $\mTheta$, given a set of $\nN$ image samples $\{\ii{\vx}{i}\}_{i=1}^\nN$ from $K$ classes
and the corresponding labels~$\{\ii{\vy}{i} \in \{1, \cdots, K\} \}_{i=1}^\nN$,
channel pruning aims to find binary channel masks $\vm_l$ with the following objective:
\begin{equation}
\begin{aligned}
    \label{eq:opt0}
    \min_{\mTheta,\vm_l} &\   \frac{1}{\nN} \sum_{i=1}^{\nN} \fL\left(\ii{\vx}{i}, \ii{\vy}{i} ,\mTheta, \{\vm_l\}_{l=1}^L\right)\\
    \text{s.t.}   &  \  \fT\left( \fF(\cdot), \mTheta,\{\vm_l\}_{l=1}^L)\right) \le \nTB, \ \ \vm_l \in {\{0, 1\}}^{c_l},
\end{aligned}
\end{equation}
where $\fL(\cdot)$ is the loss function (e.g., cross-entropy for classification),
$c_l$ is the channel number of the $l$-th layer,
$\fT(\cdot)$ is the FLOPs function,
and $\nTB$ is a given FLOPs constraint.

Most methods prune channels with some surrogate layer-wise criteria $\fC(\cdot)$ instead of directly solving Eq.\eqref{eq:opt0}.
Following this setting and with the goal of keeping $d_l$ channels in layer~$l$, we can express the pruning objective as follows:
\begin{equation}
\begin{aligned}
    \label{eq:opt1}
    \mathop{Optimize}_{\vm_{l}}  &\   \frac{1}{\nN} \sum_{i=1}^{\nN}   \fC(\ii{\vo}{i}_l\odot \vm_{l}) \\
    \text{s.t.}   & \  \vm_{l}\in \{0,1\}^{c_{l}},\lVert \vm_{l}\lVert_{0}= d_{l},
\end{aligned}
\end{equation}
where
$\ii{\vo}{i}_l \in \R^{c_l\times w_l \times h_l}$ is the feature map of input $\ii{\vx}{i}$ in the \mbox{$l$-th} layer, and
$\odot$ is elementwise multiplication with broadcasting.
A variety of $\fC(\cdot)$ have been proposed~\cite{luo2017thinet,zhuang2018discrimination,he2019filter,kang2020operation}, which may not be completely equivalent to solving Eq.\eqref{eq:opt0} but effective in practice.

Alternatively, we can formulate Eq.\eqref{eq:opt1} as a matrix projection problem.
For clarity, we denote the $i$-th element in a vector (or an ordered set) $\mat{a}$ as $\mat{a}(i)$.
We build an ordered set of the channel index $\sI_{l} = \{ i | \vm_l(i) = 1\}$ with $|\sI_{l}|=d_{l}$
and an indicator matrix $\mVIl = [\vel_{\sI_l(1)}, \cdots, \vel_{\sI_l(d_l)}]\in \R^{c_{l} \times d_{l} }$,
where $\vel_t \in  {\{0, 1\}}^{c_l}$ denotes the one-hot vector with its $t$-th element as one.
Therefore, Eq.\eqref{eq:opt1} can been expressed as follows:
\begin{equation}
\begin{aligned}
    \label{eq:opt2}
    \mathop{Optimize}_{\sI_{l}}  &\   \frac{1}{\nN} \sum_{i=1}^{\nN}   \fC(\mVIl\T \ii{\vo}{i}_l)\\
    \text{s.t.}   &\ \text{rank}({\mVIl})= d_{l} .
\end{aligned}
\end{equation}

For an input $\ii{\vx}{i}$, we denote the pruned feature map in layer~$l$ with $\sI_{l}$ by $ \ii{\voo}{i}_{l} = \mVIl\T \ii{\vo}{i}_l \in \R^{d_l\times w_l \times h_l} $.

%% file: method-trace.tex
\begin{figure}[b]
    \centering
    \includegraphics[clip=true,width=0.99\columnwidth]{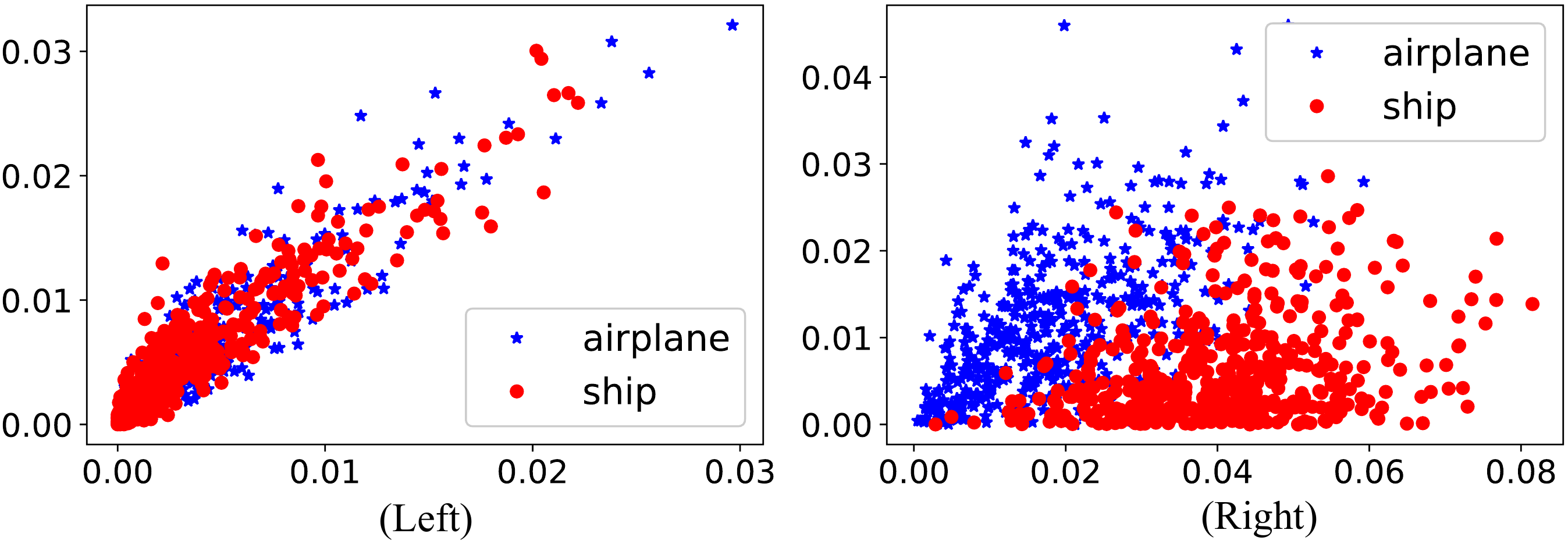}
    \caption{Visualizations of data samples from CIFAR-10 in two different feature spaces. Spaces in (Left) and (Right) are of different channel pairs in the 18-th layer of the same ResNet-20.
    }
    \label{fig:dis}
\end{figure}

In order to obtain effective criteria for model pruning, inspired by the classical Fisher Linear Discriminant Analysis (LDA)~\cite{bishop1995neural}, we dig into the feature discrimination for measuring multi-channel joint impact.
As a simple example,
Fig.~\ref{fig:dis} visualizes 
two feature spaces from the $18$-th layer in a well-trained ResNet-20 model trained on CIFAR-10.
The red and blue points represent features of images containing ships and airplanes, respectively.
We have an intuitive feeling that the features produced by channels in the right figure are more discriminative,
and a classifier built on that space is more likely to outperform that on the left space.
For deep CNNs with more layers, we hypothesis the discrimination effects accumulate and the classifier performance gap enlarges.
We describe the relation between space discrimination and classifier accuracy in Assumption~\ref{asp:1}.

\begin{assumption}
    \label{asp:1}
    For a classification task, if samples are more discriminatory in feature space $\sS_1 \in \R^{d}$ than in feature space~$\sS_2 \in \R^{d}$, it holds that the accuracy of a classifier built on $\sS_1$, namely $\fA(;\sS_1)$, is no less than that of the classifier built on $\sS_2$ with large probability, i.e.,
    $\fA(;\sS_1) \ge_{p} \fA(;\sS_2)$,
    where $\ge_{p}$ means ``no less with probability $p$''.
\end{assumption}

The above assumption enables us to concretely measure and optimize the discriminations.
Suppose $\mX = [\ii{\vx}{1}, \cdots, \ii{\vx}{N}] \in \R^{c \times w \times h \times N}$ denotes $N$ data samples from the training set with $K$ classes and $\mY = [\ii{\vy}{1}, \cdots, \ii{\vy}{N}]$ denotes the corresponding class label.
We build two
graph weight
matrices $\Gw  \in \R^{N\times N}$ and $\Gb  \in \R^{N\times N}$ to depict the prior knowledge about data sample relationship.
$\Gw $ reflects within-class relationship, where each element $\Gwij$ has a larger value if samples $\ii{\vx}{i}$ and $\ii{\vx}{j}$ belong to the same class, and smaller otherwise.
Analogically, $\Gb $ represents the between-class relationship, where $\Gbij$ is larger if samples $\ii{\vx}{i}$ and $\ii{\vx}{j}$ belong to different classes.
Without any other priors,
we directly use the Fisher score~\cite{bishop1995neural} as the graph metric for its simplicity and efficiency:
\begin{eqnarray}
    \label{eq:gw1b1}
    \Gwij &=& \left\{
                    \begin{array}{ll}
                      \frac{1}{n_{k}},&\ \ \ii{\vy}{i}=\ii{\vy}{j}=k\\  %
                      0,& \ \ \ii{\vy}{i} \neq \ii{\vy}{i}
                    \end{array}
                  \right.\\
     \label{eq:gb1}
    \Gbij &=& \frac{1}{N} - \Gwij,
\end{eqnarray}
where $n_{k}$ is the number of samples in the $k$-th class, and $\sum_{k=1}^K n_{k}=N$.

Given Assumption~\ref{asp:1} and the prior matrices in Eq.\eqref{eq:gw1b1} and Eq.\eqref{eq:gb1}, a natural criterion $\fC(\cdot)$ for channel pruning is to keep channels that reduce the discrimination of samples within the same class ($\sum_{ij}||\ii{\voo_{l}}{i}-\ii{\voo_{l}}{j}||\Gwij$) and enlarge that of samples from different classes ($\sum_{ij}||\ii{\voo_{l}}{i}-\ii{\voo_{l}}{j}||\Gbij$).
Therefore, we can replace Eq.\eqref{eq:opt2} by
\begin{equation}
   \begin{aligned}
        \label{eq:opt3}
        \max_{\sI_{l}} &\ \frac{\sum_{ij}||\mVIl\T \ii{\vo}{i}_{l}-\mVIl\T \ii{\vo}{j}_{l}||\Gbij}{\sum_{ij}||\mVIl\T \ii{\vo}{i}_{l}-\mVIl\T \ii{\vo}{j}_{l}||\Gwij}\\
        \text{s.t.} &\ |\sI_{l}|=d_{l}.
    \end{aligned}
\end{equation}

As
$\Dw$ is a diagonal matrix with $\Dwii=\sum_{j}\Gwij$ and
$\Db$ is a diagonal matrix with $\Dbii=\sum_{j}\Gbij$, $\Dw-\Gw$ and $\Db-\Gb$ are two Laplace matrices. Based on the Laplace matrix properties, Eq.\eqref{eq:opt3} is mathematically equivalent to the following trace optimization:
\begin{equation}
    \begin{aligned}
        \max_{\sI_{l}} &\ \frac{tr(\mVIl\T \mO_l(\Db-\Gb)\mO_l\T {\mVIl})}{tr(\mVIl\T \mO_l(\Dw-\Gw)\mO_l\T {\mVIl})} \label{eq:opt4}\\
        \text{s.t.} &\ |\sI_{l}|=d_{l},
    \end{aligned}
\end{equation}
where $\mO_l = [\ii{\vo}{1}, \cdots, \ii{\vo}{N}]\in \R^{c_l \times w_l \times h_l \times N}$ is the aggregated feature map for $\mX$ in layer $l$,
and $tr(\cdot)$ denotes the trace.
Equivalently,
our goal is to find the optimal channel index set~$\sI_{l}^{*}$ with the maximum trace ratio criterion for discrimination:
\begin{equation}
  \lambda_{l}^{*} = \frac{tr(\mVIlopt \T \Gbl \mVIlopt)}{tr(\mVIlopt \T \Gwl \mVIlopt)},
\end{equation}
where $\Gwl = \mO_l(\Dw-\Gw)\mO_l\T$ and $\Gbl = \mO_l(\Db-\Gb)\mO_l\T$.
Since $tr(\mVIlopt \T(\Gbl-\lambda_{l}^{*}\Gwl)\mVIlopt) = 0$, for all $\sI_{l}$ satisfying $|\sI_{l}|=d_{l}$ we have
\begin{eqnarray}
&\frac{tr(\mVIl \T \Gbl \mVIl)}{tr(\mVIl \T \Gwl \mVIl)} \leq \lambda_{l}^{*}\label{ieq:1}\\
\Leftrightarrow  &tr(\mVIl \T(\Gbl-\lambda_{l}^{*}\Gwl)\mVIl) \leq 0.\label{ieq:2}
\end{eqnarray}

Therefore, to optimize Eq.\eqref{eq:opt4} is further equivalent to optimize the following optimization:
\begin{equation}
\begin{aligned}
\max_{\sI_{l}} & \ tr(\mVIl \T(\Gbl-\lambda_{l}^{*}\Gwl)\mVIl) \label{eq:opt5}\\
s.t. &\ |\sI_{l}|=d_{l}.
\end{aligned}
\end{equation}

To be more clear, the objective in Eq.\eqref{eq:opt4} is the left-hand side terms of the inequation in Eq.\eqref{ieq:1}, and the optimal solution~$I_{l}^{*}$ of Eq.\eqref{eq:opt4} makes the equality of the inequation in Eq.\eqref{ieq:1} hold (by the definition of $\lambda_l^*$).
And similarly, given $\lambda_{l}^{*}$, the objective in Eq.\eqref{eq:opt5} is the left-hand side terms of the inequation in Eq.\eqref{ieq:2}, and the solution $I_l$ of Eq.\eqref{eq:opt5} makes the equality of the inequation in Eq.\eqref{ieq:2} hold.
Lastly, the equality of the inequation in Eq.\eqref{ieq:1} holds if and only if the equality of the inequation in Eq.\eqref{ieq:2} holds.
Thus, the optimal solution of Eq.\eqref{eq:opt4} makes equality of the inequation in Eq.\eqref{ieq:1} hold, makes equality of the inequation in Eq.\eqref{ieq:2} hold, and is also the optimal solution of Eq.\eqref{eq:opt5}.
In contrast, the optimal solution of Eq.\eqref{eq:opt5} makes equality of the inequation in Eq.\eqref{ieq:2} hold, makes equality of the inequation in Eq.\eqref{ieq:1} hold, and is also the optimal solution of Eq.\eqref{eq:opt4}.

Noting that solving Eq.\eqref{eq:opt5} naturally considers the multi-channel joint impacts inherited in $\mV$.

%% file: method-creterion.tex
To solve Eq.\eqref{eq:opt5}, we introduce a submodular set optimization, provide theoretical guarantees
on the pruned model accuracy, and propose a layer-wise pruning procedure.

Firstly, we introduce $\nsli = \exp ( {\vel_i}\T (\Gbl -\lambda_l \Gwl )  \vel_i )$ as the discrimination score for channel $i$ in layer $l$ given a specified ratio $\lambda_l$,
and $\fHl(\sI) = \log (\sum_{i \in \sI} \nsli )$ as the index set function given an index set $\sI \subseteq \{1, \cdots, c_l\}$.
When $\lambda_{l}^{*}=\lambda_{l}$,
we notice that Eq.\eqref{eq:opt5} is equivalent to the following one:
\begin{equation}
   \begin{aligned}
        \label{eq:opth}
        \max_{\sI_{l}} & \ \fHl(\sI_l;\lambda_{l}) \\
        \text{s.t.} &\ |\sI_{l}|=d_{l}.
   \end{aligned}
\end{equation}

We introduce two nice properties of set functions and show $\fHl(\cdot)$ satisfies them.

\begin{definition}
\label{def:sub}
(\textit{Submodularity}): A set function $f(\sS):2^{\sU}\rightarrow \R$ is submodular if for any subset $\sS_1 \subseteq \sS_2 \subseteq \sU$ and any element $i \in \sU \setminus \sS_2$, it has
$    f(\sS_1 \cup \{i\}) - f(\sS_1) \ge f(\sS_2 \cup \{i\}) - f(\sS_2)$.
\end{definition}

\begin{definition}
\label{def:mon}
(\textit{Monotonicity}): A set function $f(\sS):2^{\sU}\rightarrow \R$ is monotone if for any subset $\sS_1 \subseteq \sS_2 \subseteq \sU$ and any element~$i \in \sU \setminus \sS_2$, it has
 $   f(\sS_1) \le f(\sS_2)$.
\end{definition}

\begin{lemma}
    \label{lemma:1}
    Given a $\lambda_{l}$, $\fHl(\cdot)$ is a set function with monotonicity and submodularity.
\end{lemma}

\begin{proof}
Suppose $\sI \subseteq \sI' \subseteq \sU_{l}= \{1, \cdots, c_l\}$, where $c_l$ is the number of channels in layer $l$. For any $i\in \sU_{l}\setminus \sI'$, we have
\begin{equation*}
\begin{aligned}
\fHl(\sI') - \fHl(\sI)
=\log (1+\frac{\sum_{i\in \sI' \setminus \sI} \nsli}{\sum_{i\in \sI} \nsli})>0,
\end{aligned}
\end{equation*}
and
\begin{equation*}
\begin{aligned}
\fHl(\sI \cup \{i\}) - \fHl(\sI)
=&\log (1+\frac{\nsli}{\sum_{i\in \sI} \nsli})\\
>&\log (1+\frac{ \nsli}{\sum_{i\in \sI'} \nsli})\\
=&\fHl(\sI'\cup \{i\}) - \fHl(\sI').
\end{aligned}
\end{equation*}

Therefore, $\fHl(\cdot)$ is a set function with monotonicity and submodularity.
\end{proof}

\eat{
Next, we consider the optimization with a given $\lambda_{l}$. Suppose $I^{A}\subseteq I^{B}\subseteq U_{l}$. For any $i\in U_{l}\setminus I^{B}$, we have
\begin{equation}
\begin{aligned}
&H^{l}(I^{B}) - H^{l}(I^{A})\\
=&\log (1+\frac{\sum_{a\in I^{B}\setminus I^{A}} s_{a}^{l}}{\sum_{a\in I^{A}} s_{a}^{l}})>0,
\end{aligned}
\end{equation}
and
\begin{equation}
\begin{aligned}
&H^{l}(I^{A}\cup \{i\}) - H^{l}(I^{A})\\
=&\log (1+\frac{ s_{i}^{l}}{\sum_{a\in I^{A}} s_{a}^{l}})\\
>&\log (1+\frac{ s_{i}^{l}}{\sum_{a\in I^{B}} s_{a}^{l}})\\
=&H^{l}(I^{B}\cup \{i\}) - H^{l}(I^{B}).
\end{aligned}
\end{equation}
Therefore, $H^{l}(\cdot)$ is a set function with monotonicity (Definition \ref{def:mon}) and submodularity (Definiton \ref{def:sub}).
}

Secondly, to maximize $\fHl(\sI)$ with a fixed $\lambda_l$, we show that sequentially selecting the indexes $i$ with the largest score~$\nsli$ obtains a model with a guaranteed accuracy according to Theorem~\ref{theo:1}.

\begin{theorem}
\label{theo:1}
Given a $\lambda_{l}$, maximizing $\fHl(\cdot)$ by sequentially selecting $d_{l}$ elements with the highest scores from $\{\nsli | 1 \le i \le c_l \}$,
with $\sII_{l}$ representing the selected index set,
results a model that, with a large probability $p$, has a lower boundary on the accuracy related to the optimal index set $\sI_{l}^*$, i.e.,
\begin{equation*}
    \fA(\sII_{l};\lambda_{l}) \ge_{p} \Phi((1-\frac{1}{e})\fHl(\sI_{l}^*; \lambda_{l})),
\end{equation*}
where $\Phi(\cdot)$ is the monotone non-decreasing mapping from $H_{l}(\cdot)$ to accuracy $A(\cdot)$ under Assumption~\ref{asp:1}.
\end{theorem}

\begin{proof}
Suppose the optimal solution is $\sI_{l}^*=\{i_{1}^*,\cdots,i_{d_{l}}^*\}$, the index set obtained by sequentially selected top $d_{l}$ elements is $\sII_{l}=\{i_{1},\cdots,i_{d_{l}}\}$. Let $\sI_{l}^{t}=\{i_{1},\cdots,i_{t}\}$ represent the index set at step $t$ of the sequencital selection.
Since $\fHl(\cdot)$ is monotonic and submodular, following the similar proof~\cite{nemhauser1978analysis}, we have the inequalities:
\begin{eqnarray*}
\fHl(\sI_{l}^*)&\leq& \fHl(\sI_{l}^{t} \cup \sI_{l}^*)\\
&=& \fHl(\sI_{l}^{t}) + (\fHl(\sI_{l}^{t}\cup \{i_{1}^*\}) - \fHl(\sI_{l}^{t}))\\
&& + (\fHl(\sI_{l}^{t}\cup \{i_{1}^*,i_{2}^*\}) - \fHl(\sI_{l}^{t}\cup \{i_{1}^*\})\\
&& + \cdots \\
&&+ (\fHl(\sI_{l}^{t}\cup \sI_{l}^*) - \fHl(\sI_{l}^{t} \cup \{i_{1}^*,\cdots,i_{d_{l}-1}^*\})\\
& \leq & \fHl(\sI_{l}^{t}) + d_{l}(\fHl(\sI_{l}^{t+1})-\fHl(\sI_{l}^{t})),
\end{eqnarray*}
and then we have
\begin{eqnarray*}
&\fHl(\sI_{l}^*)-\fHl(\sI_{l}^{t+1})\leq (1-\frac{1}{d_{l}})(\fHl(\sI_{l}^*)-\fHl(\sI_{l}^{t}))\\
\Rightarrow& \fHl(\sII_{l}) \ge (1-(1-\frac{1}{d_{l}})^{d_{l}})\fHl(\sI_{l}^*)\\
\Rightarrow& \fHl(\sII_{l}) \ge (1-\frac{1}{e})\fHl(\sI_{l}^*).
\end{eqnarray*}

When $\sII_{l}$ is obtained, we can calculate a discrimination:
\begin{equation*}
  \widehat{\lambda_{l}}=\frac{tr(\mV_{\sII_{l}}\T \Gwl \mV_{\sII_{l}})}{tr(\mV_{\sII_{l}}\T \Gbl \mV_{\sII_{l}})}.
\end{equation*}

Thus, $\fHl(\sII_{l})$ can be regarded as a description of discrimination. Under the assumption 1, let $\Phi(\cdot)$ is the monotone non-decreasing mapping from $\fHl(\cdot)$ to the accuracy $\fA(\cdot)$, we have such an accuracy lower boundary with a large probability~$p$:
\begin{equation*}
    \fA(\sII_{l}) \ge_{p} \Phi((1-\frac{1}{e})\fHl(\sI_{l}^*)).
\end{equation*}

We can further express the boundary as
\begin{equation*}
    \fA(\sII_{l};\lambda_{l}) \ge_{p} \Phi((1-\frac{1}{e})\fHl(\sI_{l}^*;\lambda_{l})).
\end{equation*}
\end{proof}

Finally, we consider to iteratively optimize the discrimination ratio.
Suppose the current index set is $\sI_{l}$ and the discrimination ratio is $\lambda_{l}$.
After maximizing $\fHl(\sI_l; \lambda_{l})$ with the index sequential selection procedure, we get a new index set $\sI'_{l}$ and a new ratio $\lambda'_{l}$.
From Theorem~\ref{theo:2}, we have $\lambda'_{l} \ge \lambda_{l}$.

\begin{theorem}
\label{theo:2}
The $\lambda_{l}$ obtained by iteratively maximizing $\fHl(\cdot)$ is monotonic increasing and converges to the optimal value.
\end{theorem}

\begin{proof}
Under the specific indicator matrix
$\mVIl = [\vel_{\sI_l(1)}, \cdots, \vel_{\sI_l(d_l)}]\in \R^{c_{l} \times d_{l} }$,
where $\vel_i \in  {\{0, 1\}}^{c_l}$ denotes the one-hot vector with its $i$-th element as one.
we have the following equivalent optimizations:
\begin{eqnarray*}
& \max_{\sI_{l}} \exp\left(tr(\mV_{\sI_{l}}\T(\Gbl-\lambda_{l}\Gwl)\mV_{\sI_{l}})\right)\\
\Leftrightarrow& \max_{\sI_{l}} \sum_{i=1}^{d_{l}} (e_{\sI_{l}(i)}^{(l)})\T (\Gbl-\lambda_{l}\Gwl)e_{\sI_{l}(i)}^{(l)}\\
\Leftrightarrow&\max_{\sI_{l}} \log (\sum_{i\in \sI_{l}} \nsli)\\
\Leftrightarrow&\max_{\sI_{l}} \fHl\left(\sI_l; \lambda_{l}\right).
\label{opt6}
\end{eqnarray*}

The first equivalence relation comes from the definition of trace and $\mV_{\sI_{l}}$, the second equivalence relation comes from the monotonicity of exponential function and logarithmic function, and the last comes from the definition of $\fHl(\cdot)$.

Suppose the current index set (of the iteration round $t$) is $\sI_{l}^{t}$ and the discrimination is $\lambda_{l}^{t}$, and after maximizing $\fHl(\sI_l;\lambda_{l}^{t})$ we get the next index set~$\sI_{l}^{t+1}$.
Let
\begin{eqnarray*}
    f(\lambda_{l}^t) &=&{\max}_{\sI_{l}} \exp\left(tr(\mV_{\sI_{l}}\T(\Gbl-\lambda_{l}^t\Gwl)\mV_{\sI_{l}})\right)\\
    &=& \exp\left(tr(\mV_{\sI_{l}^{t+1}}\T (\Gbl-\lambda_{l}^t\Gwl)\mV_{\sI_{l}^{t+1}})\right)\\
    &>&0,
\end{eqnarray*}
then
\begin{align*}
f^{\prime}(\lambda_{l}^t) = -f(\lambda_{l}^t)tr(\mV_{\sI_{l}^{t+1}}\T \Gwl \mV_{\sI_{l}^{t+1}}) \le 0,
\end{align*}
which means $f(\lambda_{l}^t)$ is a monotonic non-increasing function.
In addition, $f(\lambda_{l}^*) = 1$ because of the definition of $\lambda_{l}^*$.

Let $h(\lambda_{l})$ be a piecewise logarithmic linear approximation of $f(\lambda_{l})$ at point $\lambda_{l}^{t}$, namely
\begin{eqnarray*}
    h(\lambda_{l}) = \exp\left(\frac{f^{\prime}(\lambda_{l}^{t})}{f(\lambda_{l}^{t})}(\lambda_{l} - \lambda_{l}^{t}) + \log f(\lambda_{l}^{t})\right) > 0,
\end{eqnarray*}
then we know $h(\lambda_{l}^{t})=f(\lambda_{l}^{t})$, and $h^{\prime}(\lambda_{l}) = h(\lambda_{l})\cdot \frac{f^{\prime}(\lambda_{l}^{t})}{f(\lambda_{l}^{t})} \le 0$, which means $h(\lambda_l)$ is also monotonic non-increasing.

Let $h(\lambda_{l}^{t+1})=1$, we have
\begin{equation*}
   \lambda_{l}^{t+1} = \frac{tr(\mV_{\sI_{l}^{t+1}}\T \Gbl \mV_{\sI_{l}^{t+1}})}{tr(\mV_{\sI_{l}^{t+1}}\T \Gwl \mV_{\sI_{l}^{t+1}})},
\end{equation*}
which is the same value as the discrimination score calculated with the next index set~$\sI_{l}^{t+1}$.

Since $f(\lambda_{l})$ is a monotonic non-increasing function with optimal (i.e., the largest) $\lambda_{l}^{*}$ satisfying $f(\lambda_{l}^{*}) = 1$, we have $$h(\lambda_{l}^{t})=f(\lambda_{l}^{t})>f(\lambda_{l}^{*})=1=h(\lambda_{l}^{t+1})$$
for $\forall \lambda_{l}^{t} \ne \lambda_{l}^{*}$. This shows $\lambda_{l}^{t+1} > \lambda_{l}^{t}$ when $\lambda_{l}^{t} \ne \lambda_{l}^{*}$.

In addition, the definition of $\lambda_l$ guarantees that $\lambda_l$ can only take a finite number of different values during the iterative optimization process.

Thus, the discrimination $\lambda_{l}$ obtained by iteratively maximizing $\fHl(\sI;\lambda_{l})$ is monotonic increasing until it converges to the optimal value.
\end{proof}

\begin{algorithm}[t!]
    \caption{Layer-wise Pruning with Trace Ratio Criterion}
    \label{alg1}
    {\bf Input:}
    $N$ sampled data $\mX \in \R^{c\times w \times h \times N}$ and their labels $\mY \in \{1, \cdots, K\}^N$;
    Original CNN model $\fF(;\mTheta)$ with $L$ CONV layers;
    Original and pruned channel number $\vc = [c_1, \cdots, c_L]$ and $\vd = [d_1, \cdots, d_L]$ for all layers;
    Stop criterion $\epsilon$

    {\bf Output:}
    Channel masks $\{\vm_l\}_{l=1}^L$ for all layers

    \begin{algorithmic}[1]
        \STATE Set $\mOO_0 \leftarrow \mX$
        \FOR {$l \leftarrow 1$ to $L$}
            \STATE Calculate feature maps $\mO_l$ by $\mOO_{l-1}$ and CNN $\fF(;\mTheta)$
            \STATE Randomly set channel index subset ${\sI'}_{l}$ with $|{\sI'}_{l}|=d_{l}$%
            \STATE Calculate ${\lambda'}_{l}$ with ${\sI'}_{l}$
            \REPEAT
                \STATE Update $(\sI_{l}, \lambda_{l}) \leftarrow ({\sI'}_{l}, {\lambda'}_{l})$
                \STATE Calculate $\{s_{l, i}\}_{i=1}^{c_l}$ for all channels with $\lambda_{l}$
                \STATE Set ${\sI'}_l$ by selecting $d_l$ indexes from ${\{1, \cdots, c_l\}}$ with the largest $s_{l, i}$
                \STATE Calculate ${\lambda'}_l$ with ${\sI'}_l$
            \UNTIL{${\lambda'}_l - \lambda_{l} \leq \epsilon$}
            \STATE Set mask $\vm_l$ for layer $l$ based on $\sI_{l}$
            \STATE Calculate pruned feature maps $\mOO_l$ by $\vm_l$ and $\mOO_{l-1}$
        \ENDFOR
        \STATE {\bf return} $\{\vm_l\}_{l=1}^L$
    \end{algorithmic}
\end{algorithm}

The whole layer-wise pruning procedure, given the channel budgets for each layer, is shown in Algorithm~\ref{alg1}.

Theorems~\ref{theo:2}~and~\ref{theo:1} show that, given a well-trained unpruned network and channel numbers $d_l$, the pruning procedure in Algorithm~\ref{alg1} leads to monotonically increasing feature discrimination value until reaching the optimal solution~(by Theorem~\ref{theo:2}) and obtains the pruned network with an accuracy lower boundary proportional to the best-pruned network with large probability~(by Theorem~\ref{theo:1}).

Note that the objective in each layer is not a simply linear sum of each channel. Determining the value of $\lambda_{l}$ in $\nsli$ incorporates the joint-channel impact and depends on other channels. With the update of $\lambda_{l}$, channels selected in the previous rounds may be unselected.
In other words, the set of channels is not built by step-by-step greedy inclusion.

Compared to training-based pruning,
the time complexity of our layer-wise pruning procedure is linear to the number of sampled data, which is usually much smaller than (and increased sub-linearly to) the number of original training data, and it only performs inference once and no back-propagation.
Theoretically, the convergence of the pruning step is guaranteed; Empirically, it only uses three to four iterations to converge to the pruned network structures, taking a negligible amount of time.

%% file: method-framework.tex
We further propose an approximate algorithm to determine the architecture of the pruned model, i.e., to find the optimal channel numbers $\vd$, as a prerequisite of the layer-wise pruning stage shown in Algorithm~\ref{alg1}.
Besides the discrimination criteria, we also consider the FLOPs constraint, aiming to quickly estimate the effects of local channel number changes on the whole model's power and efficiency.

\begin{algorithm}[t!]
    \caption{Greedy Selection for Channel Numbers}
    \label{alg2}
    {\bf Input:}
    $N$ sampled data $\mX \in \R^{c\times w \times h \times N}$ and their labels $\mY \in [1, \cdots, K]^N$;
    Original CNN model $\fF(;\mTheta)$ with $L$ CONV layers;
    Initial channel number $d_{\text{min}}$;
    Step size $\eta$;
    FLOPs constraint $\nTB$

    {\bf Output:}
    Channel numbers $\vd = [d_1, \cdots, d_L]$ for all layers

    \begin{algorithmic}[1]
        \STATE Calculate feature maps $\{\mO_l\}_{l=1}^L$ for all layers with $\mX$
        \STATE Set all $L$ values in $\vd$ by $d_{\text{min}}$
        \STATE Randomly select $d_l$ channels in each layer and calculate $\{\lambda_l\}_{l=1}^L$ accordingly
        \STATE Calculate $\{\fDT_l(d_l)\}_{l=1}^L$ by Eq.\eqref{eq:delta} with $\{\mO_l\}_1^L$ and $\mY$
        \STATE Set $l_{\text{max}}  \leftarrow  \arg \max_{l} \fDT_l(d_l)$
        \STATE Set $d'  \leftarrow  d_{l_{\text{max}}}+ \eta$
        \WHILE {$\sum_{l\neq l_{\text{max}} } \fT_{l}(d_{l}) + \fT_{l_{\text{max}}}(d') \le \nTB$}
            \STATE Update $d_{l_{\text{max}}} \leftarrow d'$
            \STATE Update $\sI_{l_{\text{max}}}$ and $\lambda_{l_{\text{max}}}$ based on $d_{l_{\text{max}}}$
            \STATE Update $\fDT_{l_{\text{max}}} ( d_{l_{\text{max}}} )$ and other changed terms in $\fDT_{l}(d_l)$
            \STATE Set $l_{\text{max}}  \leftarrow  \arg \max_{l} \fDT_l(d_l)$
            \STATE Set $d'  \leftarrow  d_{l_{\text{max}}}+ \eta$
        \ENDWHILE
        \STATE {\bf return} $\vd = [d_1, \cdots, d_L]$
    \end{algorithmic}
\end{algorithm}

We first introduce a network discrimination approximation function with given $\{\lambda_{l}\}_{l=1}^{L}$, which is
$    \fD(\vd) = \sum_{l=1}^L \max_{|\sI_l|=d_l} \fHl(\sI_l; \lambda_l)$.
Taking the suboptimal solution in Theorem~\ref{theo:1} and ignoring the effects of $d_l$ in other layers, we have its derivative as follows:
\begin{equation}
\frac{\pfD(\vd)}{\partial d_l}
\approx \log\left(1+\frac{{\tilde{s}_{l, d_l+1}}}{\sum_{i=1}^{d_{l}}\tilde{s}_{l,i}}\right)
\approx  \frac{\tilde{s}_{l, d_{l}+1}}{\sum_{i=1}^{d_{l}}\tilde{s}_{l, i}} ,
\end{equation}
where $[\tilde{s}_{l, 1},\cdots,\tilde{s}_{l, c_l}]$ is the descending sorted elements of $\{\nsli \}_{i=1}^{c_l}$.
The two approximations come from the sequentially selection of element $d_l$ in Theorem~\ref{theo:1} and Taylor expansion, respectively.

On the other hand, to measure the network computational efficiency,
we define the FLOPs function of layer $l$ as
$\fT_l(\vd) = d_{l-1} d_l w_l h_{l} q_{l}^{2}$,
where $q_{l} \times q_{l}$, $w_l$, and $h_{l}$ is the $l$-th layer's kernel size, feature map width, and feature map height, respectively.
The FLOPs function for the entire network is
$\fT(\vd) = \sum_{l=1}^L \fT_l(\vd)$,
and its derivative for $d_i$ is as follows:
\begin{equation}
\frac{ \pfT(\vd)}{\partial d_l} = d_{l-1}w_lh_l{q_l^2} + d_{l+1}w_{l+1}h_{l+1}{q_{l+1}^2}.
\end{equation}

As a tradeoff to jointly considering computations and discriminations, we define the following function to evaluate the effect of changing each channel number $d_l$:

\begin{equation}
\fDT(d_l) = \frac{\pfD}{\pfT} .\label{eq:delta}
\end{equation}

For the global architecture optimization, we manipulate the channel numbers without selecting the channels.
Therefore, we take a greedy coordinate ascent strategy to find the optimal channel numbers, which is describe in Algorithm~\ref{alg2}.
It is worth noting that, similar to the layer-wise pruning procedure, this algorithm only performs inference once without backpropagation. After that, it only needs to recalculate a small portion of variables in each iteration.

%% file: method-all.tex
\begin{figure}[t]
    \centering
    \includegraphics[width=0.98\columnwidth]{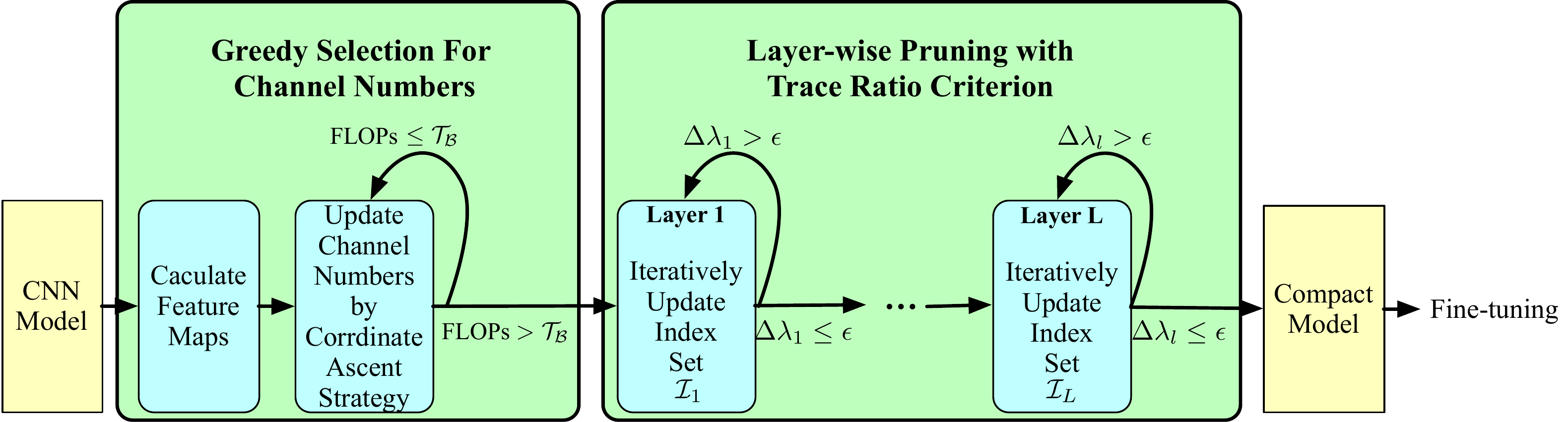}
    \caption{The pipeline of the overall pruning procedure in {\catro}.}
    \label{fig:pipeline}
\end{figure}

We take the two stages together to build {\catro}, the proposed channel pruning method.
As illustrated in Fig.~\ref{fig:pipeline},
{\catro} first employs Algorithm~\ref{alg2}~(Greedy Selection for Channel Numbers) to select the channel numbers for each layer and then uses Algorithm~\ref{alg1}~(Layer-wise Pruning with Trace Ratio Criterion) to prune each layer.
It fine-tunes the pruned network as the last step on all available training data, similarly to other pruning methods.

%% file: exp.tex
\input{exp-result-tab1.tex}

\subsection{Datesets and Implementation Details}

\input{exp-dataset.tex}

\input{exp-implementation.tex}

\subsection{Quantitative Performance Comparisons}

\input{exp-results-quant.tex}

\subsection{Investigations on Determining Channel Numbers}
\input{exp-results-visual.tex}

\subsection{Results on Classification Subtask Compression}
\label{subsec:subtasks}
\input{exp-results-subtask.tex}

\subsection{Ablation Studies}
\input{exp-ablation-alg2.tex}
\input{exp-ablation-layer-group.tex}

\subsection{Discussions}
\input{exp-discussion.tex}

%% file: exp-result-tab1.tex
\begin{table*}[t!]
    \centering
    \caption{Comparisons of the pruned ResNet networks on CIFAR-10 and CIFAR-100.}
    \resizebox{.85\textwidth}{!}
    {
        \begin{tabular}{c c|c c c| c c c}
            \toprule
            \multirow{2}{*}{Depth}& \multirow{2}{*}{Method}& \multicolumn{3}{c|}{CIFAR-10} & \multicolumn{3}{|c}{CIFAR-100} \\
             & & Accuracy & Acc. Drop $\downarrow$ & FLOPs (Drop $\downarrow$) &  Accuracy & Acc. Drop $\downarrow$ & FLOPs (Drop $\downarrow$)\\
            \midrule \midrule

            \multirow{9}{*}{20}
            & LCCL~\cite{dong2017more}  & 91.68\% & 1.06\%  & 2.61E7 (36.0\%) & 64.66\% & 2.87\% & 2.73E7 (33.1\%) \\
            & SFP~\cite{he2018soft}  & 90.83\% & 1.37\%  & 2.43E7 (42.2\%) & 64.37\% & 3.25\% & 2.43E7 (42.2\%) \\
            & FPGM~\cite{he2019filter}  & 91.09\% & 1.11\%  & 2.43E7 (42.2\%) & 66.86\% & 0.76\% & 2.43E7 (42.2\%) \\
            & CNN-FCF~\cite{li2019compressing} & 91.13\% & 1.07\% & 2.38E7 (41.6\%) & - & - & - \\
            & Taylor FO~\cite{molchanov2019importance} & 91.52\% & 0.48\% & 2.83E7 (32.7\%) & - & - & -\\
             & PScratch~\cite{WangZXZSZH2020Pruning} & 90.55\% & 1.20\% & 2.06E7 (50.0\%) & - & - & - \\
             & DHP~\cite{LiGZGT20DHP} & 91.54\% & 1.00\% & 2.14E7 (48.2\%) & - & - & - \\
             & SCOP~\cite{Tang2020SCOP} & 90.75\% & 1.44\% & 1.83E7 ({55.7}\%) & - & - & - \\
            \cmidrule(r){2-8}
            & Ours & \textbf{91.76\%} & 0.98\% & 1.95E7 ({52.7\%})  & \textbf{68.61\%} &  0.99\% & 2.02E7 (\textbf{51.0\%})\\
            \midrule

            \multirow{10}{*}{32}
            & LCCL~\cite{dong2017more}  & 90.74\% & 1.59\%  & 4.76E7 (31.2\%) & 67.39\% & 2.69\% & 4.32E7 (37.5\%) \\
            & SFP~\cite{he2018soft}  & 92.08\% & 0.55\%  & 4.03E7 (41.5\%) & 68.37\% & 1.40\% & 4.03E7 (41.5\%) \\
            & TAS~\cite{dong2019network}  & 93.16\% & 0.73\%  & 3.50E7 (49.4\%) & 72.41\% & -1.80\% & 4.25E7 (38.5\%) \\
            & FPGM~\cite{he2019filter}  & 92.31\% & 0.32\%  & 4.03E7 (41.5\%) & 68.52\% & 1.25\% & 4.03E7 (41.5\%) \\
            & CNN-FCF~\cite{li2019compressing} & 92.18\% & 0.25\% & 3.99E7 (42.2\%) & - & - & - \\
            & LFPC~\cite{he2020learning}  & 92.12\% & 0.51\%  & 3.27E7 (49.4\%) & - & - & - \\
             & SCOP~\cite{Tang2020SCOP}  & 92.13\% & 0.53\%  & 3.00E7 (55.8\%) & - & - & - \\
            & LFPC~\cite{he2020learning}  & 92.12\% & 0.51\%  & 3.27E7 (49.4\%) & - & - & - \\
            & DCPH~\cite{chen2021dynamical}  & 92.85\% & -0.49\%  & 4.82E7 (30.0\%) & 69.51\% & -0.63\% & 4.82E7(30.1\%) \\
             & MFP~\cite{he2022filter}  & {91.85\%} & {0.78\%}  & {3.23E7 (53.2\%)} & {-} & {-} & {-} \\
            \cmidrule(r){2-8}
            & Ours & \textbf{93.17\%}& 0.77\%& 2.98E7 (\textbf{56.1\%}) & 71.71\% & 0.24\% & 3.81E7 (\textbf{43.8\%}) \\

            \midrule
            \multirow{18}{*}{56}
            & LCCL~\cite{dong2017more}  & 92.81\% & 1.54\%  & 7.81E7 (37.9\%) & 68.37\% & 2.96\% & 7.63E7 (39.3\%) \\
            & SFP~\cite{he2018soft}  & 93.35\% & 0.24\%  & 5.94E7 (52.6\%) & 68.79\% & 2.61\% & 5.94E7 (52.6\%) \\
            & AMC~\cite{he2018amc}  & 91.90\% & 0.90\%  & 6.29E7 (50.0\%) & - & - & -  \\
            & DCP~\cite{zhuang2018discrimination}  & 93.49\% & 0.31\%  & 6.27E7 (50.0\%) & - & - & - \\
            & TAS~\cite{dong2019network}  & 93.69\% & 0.77\%  & 5.95E7 (52.7\%) & 72.25\% &  0.93\% & 6.12E7 (51.3\%) \\
            & FPGM~\cite{he2019filter}  & 93.49\% & 0.10\%  & 5.94E7 (52.6\%) & 69.66\% & 1.75\% & 5.94E7 (52.6\%) \\
            & CNN-FCF~\cite{li2019compressing} & 93.38\% & -0.24\% & 7.20E7 (42.8\%) & - & - & - \\
            & GAL~\cite{lin2019towards}  & 93.38\% & 0.12\%  & 7.83E7 (37.6\%) & - & - & - \\
            & CCP~\cite{peng2019collaborative}  & 93.50\% & 0.04\%  & 6.73E7 (47.0\%) & - & - & - \\
            & LFPC~\cite{he2020learning}  & 93.24\% & 0.35\%  & 5.91E7 (52.9\%) & 70.83\% & 0.58\% & 6.08E7 (51.6\%) \\
            & SCP~\cite{kang2020operation}  & 93.23\% & 0.46\%  & 6.10E7 (51.5\%) & - & - & - \\
            & HRank~\cite{lin2020hrank}  & 93.17\% & 0.09\%  & 6.27E7 (50.0\%) & - & - & - \\
            & DHP~\cite{LiGZGT20DHP} & 93.58\% & -0.63\% & 6.47E7 (49.0\%) & - & - & - \\
            & PScratch~\cite{WangZXZSZH2020Pruning} & 93.05\% & 0.18\% & 6.35E7 (50.0\%) & - & - & - \\
            & LeGR~\cite{ChinDZM2020Towards} & 93.70\% & 0.20\% & 5.89E7 (53.0\%) & - & - & - \\
            & {DCPH}~\cite{chen2021dynamical}  & {93.26\%} & {-0.60\%}  & {8.76E7 (30.2\%)} & {71.31\%} & {-0.41\%} & {8.78E7(30.0\%)} \\
            & {MFP}~\cite{he2022filter}  & {93.56\%} & {0.03\%}  & {5.94E7 (52.6\%)} & {-} & {-} & {-} \\
            \cmidrule(r){2-8}
            & Ours & \textbf{93.87\%} & 0.03\%& 5.83E7 (\textbf{54.0\%}) & 72.13\% &  0.67\% & 5.64E7 (\textbf{55.5\%}) \\
            \midrule

            \multirow{12}{*}{110}
            & LCCL~\cite{dong2017more}  & 93.44\% & 0.19\%  & 1.68E8 (34.2\%) & 70.78\% & 2.01\% & 1.73E8 (31.3\%) \\
            & SFP~\cite{he2018soft}  & 93.86\% & -0.18\%  & 1.21E8 (52.3\%) & 71.28\% & 2.86\% & 1.21E8 (52.3\%) \\
            & TAS~\cite{dong2019network}  & 94.33\% & 0.64\%  & 1.19E8 (53.0\%) & 73.16\% &  1.90\% & 1.20E8 (52.6\%) \\
            & FPGM~\cite{he2019filter}  & 93.85\% & -0.17\%  & 1.21E8 (52.3\%) & 72.55\% & 1.59\% & 1.21E8 (52.3\%) \\
            & CNN-FCF~\cite{li2019compressing} & 93.67\% & -0.09\% & 1.44E8 (43.1\%) & - & - & - \\
            & GAL~\cite{lin2019towards}  & 92.74\% & 0.76\%  & 1.30E8 (48.5\%) & - & - & - \\
            & LFPC~\cite{he2020learning}  & 93.07\% & 0.61\%  & 1.01E8 (60.0\%) & - & - & - \\
            & HRank~\cite{lin2020hrank}  & 93.36\% & 0.87\%  & 1.06E8 (58.2\%) & - &  - & - \\
            & DHP~\cite{LiGZGT20DHP} & 93.39\% & -0.06\% & 1.62E8 (36.3\%) & - & - & - \\
            & PScratch~\cite{WangZXZSZH2020Pruning} & 93.69\% & -0.20\% & 1.53E8 (40.0\%) & - & - & - \\
            & {DCPH}~\cite{chen2021dynamical}  & {94.11\%} & {-0.73\%}  & {1.77E8 (30.1\%)} & {72.79\%} & {-0.26\%} & {1.77E8(30.0\%)} \\
            & {MFP}~\cite{he2022filter}  & {93.31\%} & {0.37\%}  & {1.21E8 (52.3\%)} & {-} & {-} & {-} \\
            \cmidrule(r){2-8}
            & Ours & \textbf{94.41\%} & 0.41\% & 1.00E8 (\textbf{60.8\%}) & \textbf{74.04\%} & 0.45\% & 1.12E8 (\textbf{56.1\%}) \\
            \bottomrule
        \end{tabular}
    }
    \label{tab:cifar}
\end{table*}

%% file: exp-dataset.tex
\paragraph{Datasets}
We conducted our experiments on four datasets: CIFAR-10, CIFAR-100~\cite{krizhevsky2009learning}, ImageNet~\cite{deng2009imagenet}, and GTSRB~\cite{Houben-IJCNN-2013}.
Both CIFAR-10 and CIFAR-100 contain 50K training images and 10K test images, and they have 10 and 100 classes, respectively.
ImageNet contains 1.2M training images and 50K test images for 1000 classes.
GTSRB, which is used to evaluate pruning performance for classification subtasks, contains 39,209 training images and 12,630 test images of 43 types of traffic signs.

%% file: exp-implementation.tex
\paragraph{Implementation details}
For CIFAR-10/100 and GTSRB, we trained the original model 300 epochs using SGD with momentum 0.9, weight decay parameter 0.0005, batch size 256, and the initial learning rate of 0.1 dropped the factor of 0.1 after the 80th, 150th and 250th epochs. The step size is 1.
We finetuned 200 epochs for pruned models with an initial learning rate of 0.01 dropped by the factor of 0.1 after half epochs for fair comparisons. The models are trained on GPU Tesla P40 using Pytorch.
For quantitative performance comparisons, all our models do not prune the first convolution layer, and the minimal number of channels is three as the same with the input images. The numbers of channels are directly set to the minimum values. For the CIFAR-100 classification subtask compression, we moved about $3\sim5$ epochs from finetune phrase to the pruning phrase after pruning each block while keeping the total number of backward the same. In order to reduce overfitting, we limited the maximal number of channels with half of the original model. For the GTSRB classification subtask compression, we further limited the maximal number of channels with $\frac{3}{8}$ of the original model for the ultra-high compression ratios. For all subtask compression experiments, to get the high compression ratio, the first convolution layer was pruned in the same way as the other layers.
For the ImageNet, we trained 90 epochs, and the learning rate was decayed by the factor of 0.1 after the 30th, 60th epochs.

\paragraph{Hyperparameters}
The key hyperparameters are $d_{min}$ and number of data samples. We studied different numbers of samples on CIFAR-10/100. Results showed that CATRO are robust within a large range of sample number. We chose the value of $d_{min}$ with the following considerations.  Firstly, $d_{min}$ should be a positive value to avoid cutting off the network or layer-collapse issues. Secondly, $d_{min}$ should be small enough to have a large search space. We empirically set the $d_{min}$ to be 3, and we set all layers' as the same value for ease of use. We found out it may influence some layers’ pruning ratios, but the overall compression ratios are similar with almost the same accuracy performance. In practice, the minimum initial channel number $d_{min}$  is not sensitive to the overall channel number selection algorithm since most pruned layers are larger than the $d_{min}$.

%% file: exp-results-quant.tex
\paragraph{ResNet on CIFAR}
We compared the classification performance of our proposed method with several state-of-the-art channel pruning approaches on CIFAR-10 and CIFAR-100 with ResNet-20/32/56/110~\cite{kaiming2016deep}.
As different baselines use different training efforts or strategies, the original accuracy values can be different.
We reported both the pruned accuracy and the accuracy drop from all baselines' original papers to draw fair comparisons.
Results in Table~\ref{tab:cifar} clearly demonstrate the superiority of our {\catro} over other channel pruning approaches.
For CIFAR-10, {\catro} consistently outperformed other channel pruning algorithms while achieving or approaching the highest FLOPs reduction ratios in ResNet-20/32/56/110.
In the case of ResNet-110, {\catro} reduced $60.8\%$ FLOPs with a notably high accuracy $94.41\%$ and reduced $56.1\%$ FLOPs with the highest accuracy $74.04\%$ on CIFAR-10 and CIFAR-100, respectively.

\begin{table}[h]
\centering
\caption{Comparisons of the pruned MobileNet-v2 on CIFAR-10.}
\resizebox{0.85\columnwidth}{!}
{
    \begin{tabular}{c| c c c}
            \toprule
            Method
            & Accuracy & Acc. Drop $\downarrow$ & FLOPs $\downarrow$\\
            \midrule \midrule
            DCP~\cite{zhuang2018discrimination}  & 94.02\% & 0.45\% & 26.4\% \\
            SCOP~\cite{Tang2020SCOP} & 94.24\% & 0.24\% & 40.3\% \\
            {DNCP}~\cite{zheng2022model} & {93.71\%} & {0.44\%} & {41.6\%} \\
            \cmidrule(lr){1-4}
            Ours & \textbf{94.27}\% & 0.20\% & \textbf{41.6}\% \\
            \bottomrule
        \end{tabular}
}
\label{tab:mobile}
\end{table}

\paragraph{MobileNet on CIFAR}
We further validated {\catro} on the compact MobileNet-v2 models.
Results are present on Table~\ref{tab:mobile}, and {\catro} reduced $41.6\%$ FLOPs with a notably high accuracy $94.27\%$.

\begin{table*}[t]
    \centering
    \caption{Comparisons of the pruned ResNet-34/50 networks on ImageNet.}
    \resizebox{0.7\textwidth}{!}
    {
        \begin{tabular}{c c| c c| c c| c}
            \toprule
            \multirow{2}{*}{Depth}& \multirow{2}{*}{Method}& \multicolumn{2}{c|}{Top-1} & \multicolumn{2}{c|}{Top-5}& \multirow{2}{*}{FLOPs Drop $\downarrow$}\\
             & & Accuracy & Acc. Drop $\downarrow$ & Accuracy & Acc. Drop $\downarrow$ &\\
            \midrule \midrule
            \multirow{4.5}{*}{34}&SFP~\cite{he2018soft} & 71.83\% & 2.09\% & 90.33\% & 1.29\% & 41.1\%  \\
            &FPGM~\cite{he2019filter} & 72.54\% & 1.38\% & 91.13\% & 0.49\% & 41.1\%  \\
            &ABCPruner~\cite{lin2020channel} & 70.98\% & 2.30\% & - & - & 41.0\% \\
            \cmidrule(r){2-7}
            &Ours & \textbf{72.75\%} & 0.50\% & \textbf{91.13}\% & 0.41\% & \textbf{42.9\%}  \\
            \midrule
            \multirow{6.5}{*}{50}&FPGM~\cite{he2019filter} & 75.59\% & 0.56\% & 92.27\% & 0.24\% & 42.2\%  \\
            &HRank~\cite{lin2020hrank} & 74.98\% & 1.17\% & 92.33\% & 0.51\% & 43.7\% \\
            &SCOP~\cite{Tang2020SCOP} & 75.95\% & 0.40\% & 92.70\% & 0.02\% & 45.3\% \\
            &LeGR~\cite{ChinDZM2020Towards} & 75.70\% & 0.20\% & 92.79\% & 0.20\% & 42.0\% \\
            &CNN-FCF~\cite{li2019compressing} & 75.68\% & 0.47\% & 92.68\% & 0.19\% & 46.0\% \\
            &{MFP}~\cite{he2022filter} & {75.67\%} & {0.48\%} & {92.81\%} & {0.06\%} & {42.2\%} \\
            \cmidrule(r){2-7}
            &Ours & \textbf{75.98\%} & 0.14\% & 92.79\% & 0.08\% & 45.8\%  \\
            \bottomrule
        \end{tabular}
    }
    \label{tab:imagenet}
\end{table*}

\paragraph{ResNet on ImageNet}
Results on ImageNet supported our findings on CIFAR, and we show the results of ResNet-34/50 in Table~\ref{tab:imagenet}, where {\catro} achieved the best accuracies with better or similar FLOPs drops.
For ResNet-34, with fewer FLOPs ($42.9\%$ FLOPs reduction), {\catro} achieved better classification accuracies compared to others. For ResNet-50, {\catro} achieved better or competitive classification accuracies compared to many state-of-the-art methods.

%% file: exp-results-visual.tex
\begin{figure}[h]
\centering
    \includegraphics[width=0.48\textwidth]{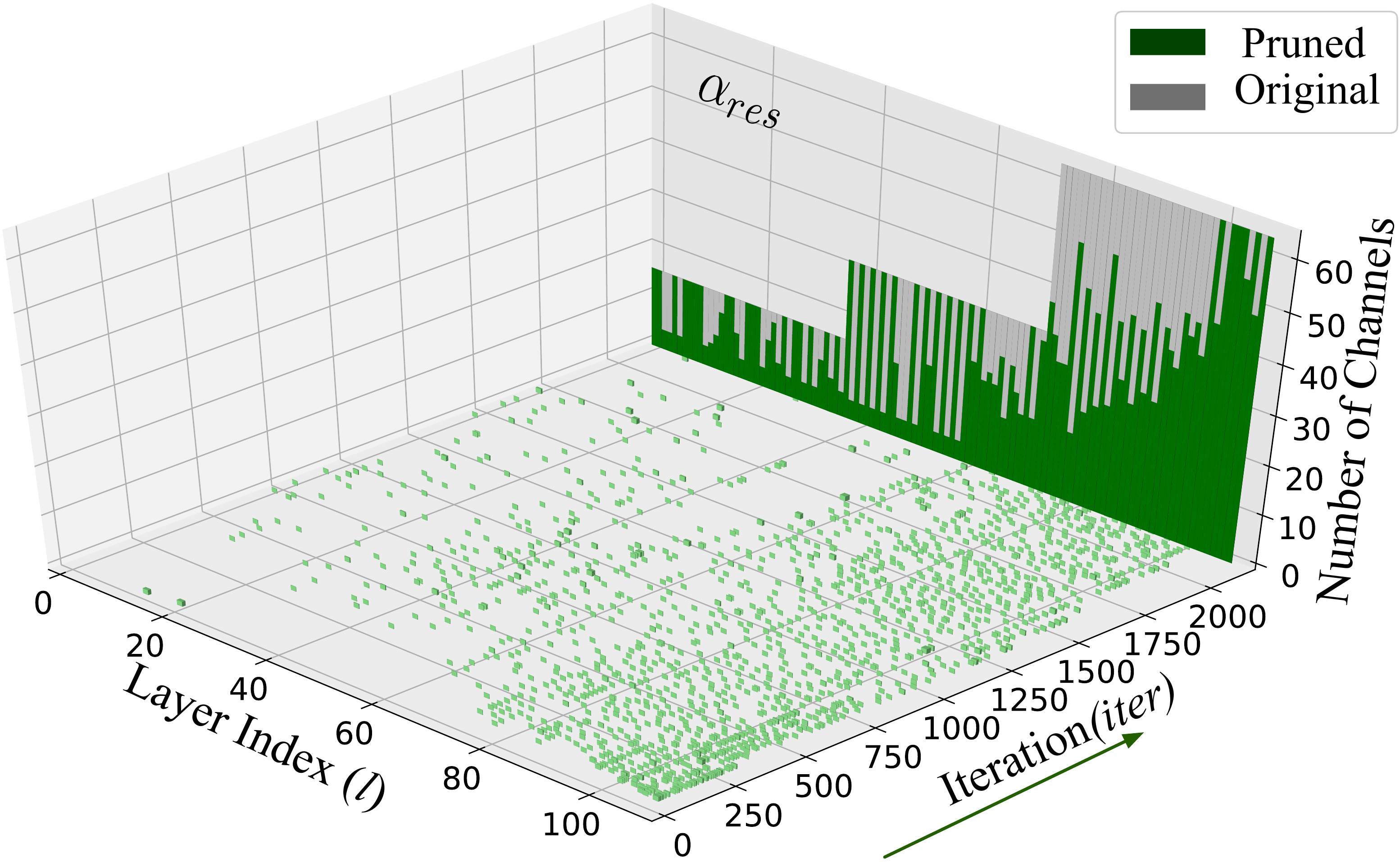}
    \caption{An illustration of the retaining channel number optimization procedure of ResNet-110 on CIFAR-10.}
    \label{fig:path}
\end{figure}

Fig.~\ref{fig:path} showcases the procedure of greedy architecture optimization (described in Section~\ref{sec:method-greedy}) for ResNet-110 on CIFAR-10 for better investigations of {\catro}.
A green point at $(iter,l,0)$ represents that at the $iter$-th iteration in Algorithm~\ref{alg2} our method decided to increase the retaining number of the $l$-th layer by one.
The channel numbers of the pruned and original network are visualized on the $\alpha_{res}$ plane in green and gray, respectively.
As shown in the figure, the proposed optimization method focused more on deeper layers (i.e., layers close to the output) at the beginning (e.g., the first 300 iterations) and gradually expanded other layers, which is consistent with the state that features in neural networks have a tendency to transfer from general to special~\cite{yosinski2014transferable,hu2017fast}.
After 2126 iterations, the optimization procedure ended up with a config of network architecture with the desired FLOPs (1.0E8).
The resulted network had 21 narrow bottleneck layers (with $\le 5$ channels) in the first half of the network and 37 fully recovered layers.

%% file: exp-results-subtask.tex
\begin{table*}[t]
\caption{Comparisons of subtask classification performance of the pruned ResNet-32/56 on CIFAR-100.}
    \centering
    \resizebox{.95\textwidth}{!}{
        \begin{tabular}{c c| c c| c c| c c}
            \toprule
            \multirow{2}{*}{Depth}& \multirow{2}{*}{Method}& \multicolumn{2}{c|}{5-Class} & \multicolumn{2}{c|}{10-Class}& \multicolumn{2}{c}{20-Class}\\
             & & Acc. (Drop $\downarrow$) & FLOPs (Drop $\downarrow$) &Acc. (Drop $\downarrow$) & FLOPs (Drop $\downarrow$)& Acc. (Drop $\downarrow$) & FLOPs (Drop $\downarrow$)\\
            \midrule \midrule
            \multirow{3.5}{*}{32}
            & Direct~\cite{li2016pruning} & 91.40\% (1.80\%) & 1.11E7 (83.63\%) & 91.00\% (1.70\%) & 1.11E7 (83.63\%) & 83.05\% (3.70\%) & 1.11E7 (83.63\%) \\
            & FPGM~\cite{he2019filter} & 91.80\% (1.60\%) & 1.11E7 (83.63\%) & 91.40\% (1.30\%) & 1.11E7 (83.63\%) & 83.05\% (3.70\%) & 1.11E7 (83.63\%) \\
            \cmidrule(r){2-8}
            & Ours & \textbf{93.60\%} (-0.40\%) & 1.09E7 (\textbf{83.92\%}) & \textbf{92.00\%} (0.70\%) & 1.07E7 (\textbf{84.22\%}) & \textbf{84.45\%} (2.30\%) & 1.08E7 (\textbf{84.07\%}) \\
            \midrule
            \multirow{3.5}{*}{56}
            & Direct~\cite{li2016pruning} & 94.00\% (1.20\%) & 2.02E7 (84.07\%) & 92.70\% (2.10\%) & 2.02E7 (84.07\%) & 85.95\% (1.80\%) & 2.02E7 (84.07\%) \\
            & FPGM~\cite{he2019filter} & 94.20\% (1.00\%) & 2.02E7 (84.07\%) & 92.90\% (1.90\%) & 2.02E7 (84.07\%) & 86.15\% (1.60\%) & 2.02E7 (84.07\%) \\
            \cmidrule(r){2-8}
            & Ours & \textbf{95.00\%} (0.20\%) & 1.95E7 (\textbf{84.62\%}) & \textbf{93.90\%} (0.90\%) & 1.97E7 (\textbf{84.46\%}) & \textbf{86.55\%} (1.20\%) & 1.98E7 (\textbf{84.38\%}) \\
            \bottomrule
        \end{tabular}
    }
    \label{tab:few}
\end{table*}

We evaluated image classification subtask compressions to demonstrate the efficacy of {\catro} for both general purposes (CIFAR-100) and practical applications (GTSRB).

\paragraph{CIFAR-100}
CIFAR-100 is a benchmark dataset for general classification with 100 classes. Without losing generality, we directly selected the top 5/10/20 classes as subtasks of interests and pruned tiny networks for each specific subtask. We compared {\catro} with a data-independent channel pruning method (`Direct', which uses that $l_{2}$-norm~\cite{li2016pruning}) and a training-based channel pruning method (FPGM~\cite{he2019filter}).
As shown in Table~\ref{tab:few}, with a high compression ratio (84\% FLOPs drop), our method consistently outperformed others with a distinct gap in all settings.

\begin{table*}[t]
\caption{Comparisons of subtask classification performance of the pruned ResNet-20 on GTSRB.}
    \centering
    \resizebox{0.85\textwidth}{!}{
        \begin{tabular}{c c| c c c| c c c}
            \toprule
            \multirow{2}{*}{Drop Ratio}&\multirow{2}{*}{Method}& \multicolumn{3}{c|}{Limit Speed Signs} & \multicolumn{3}{c}{Direction Signs}\\
             & & Accuracy & Acc. Drop $\downarrow$ & FLOPs (Drop $\downarrow$) & Accuracy & Acc. Drop $\downarrow$ & FLOPs (Drop $\downarrow$)\\
            \midrule \midrule
            \multirow{2.5}{*}{$\approx$97\%}
            &Direct~\cite{li2016pruning} & 95.75\% & 4.00\% & 9.60E5 (97.67\%) & 98.19\% & 1.02\% & 9.60E5 (97.67\%)  \\
            \cmidrule(r){2-8}
            &Ours & \textbf{96.97\%} & 2.98\% & 9.40E5 (\textbf{97.72\%}) & \textbf{98.59\%} & 0.62\% & 9.08E5 (\textbf{97.80\%}) \\
            \midrule
            \multirow{2.5}{*}{$\approx$90\%}
            &Direct~\cite{li2016pruning} & 99.23\% & 0.52\% & 4.08E6 (90.10\%) & 98.47\% & 0.74\% & 4.08E6 (90.10\%)  \\
            \cmidrule(r){2-8}
            &Ours & \textbf{99.39\%} & 0.36\% & 3.71E6 (\textbf{90.75\%}) & \textbf{99.04\%} & 0.17\% & 3.79E6 (\textbf{90.80\%}) \\
            \bottomrule
        \end{tabular}
    }
    \label{tab:gtsrb}
\end{table*}

\paragraph{GTSRB}

\begin{figure}[h]
\centering
    \includegraphics[clip=true,width=\columnwidth,viewport=0 0 700 290]{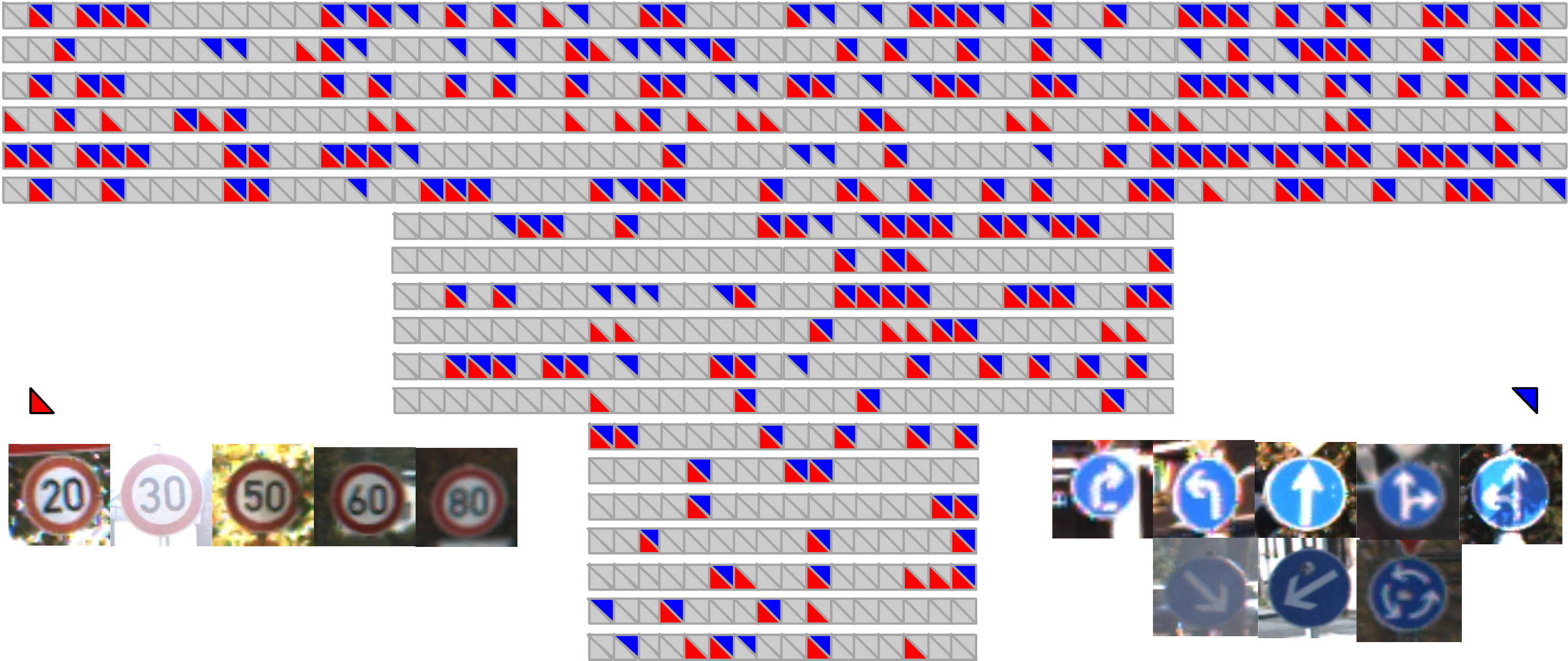}
    \caption{Channel visualizations of the pruned ResNet-20 for limit-speed-sign and direction-sign classification subtasks on GTSRB.
    The red and blue triangles indicate the preserved channels for the limit-speed-sign and direction-sign classification, respectively.
    }
    \label{fig:task}
\end{figure}

We used GTSRB to demonstrate the efficacy of {\catro} for practical usages, e.g., in intelligent traffic and driver-assistance systems.
To mimic the real-world application scenarios, we chose two subtasks from GTSRB (limit-speed-sign recognition and direction-sign recognition) and evaluated the pruning methods with two ultra-high compression ratios (90\% and 97\%).
Table~\ref{tab:gtsrb} shows the results of pruned ResNet-20 and clearly indicates the better performance of {\catro} in terms of both the accuracy and the model size.
Fig.~\ref{fig:task} visualizes the pruned ResNet-20 models from {\catro} for different tasks.
As shown in the figure, we can find out that different channels contributed to different tasks, and {\catro} obtained specific pruned models for each task by selecting different channels.
To be more specific, some of the selected/preserved channel indexes in Fig.~\ref{fig:task} are the same (the squares filled with both red and blue) while each subtask has its own specific channels (the squares filled with only red or only blue).

\begin{table}[t]
\centering
\caption{Comparisons for task exchange on GTSRB of the pruned ResNet-20 with 90\% pruning ratio.}
\resizebox{1\columnwidth}{!}
{
\begin{tabular}{c | c c c}
\toprule
 \multirow{2}{*}{Name} & Pruning Task & Target Task & \multirow{2}{*}{Accuracy} \\
 &  (Algorithms~\ref{alg1}~$\&$~\ref{alg2}) & (Finetuning) & \\
\midrule \midrule
 LSS-DS & Limit Speed Signs& Direction Signs & 98.53\% \\
 DS-DS (Ours) & Direction Signs & Direction Signs & \textbf{99.04\%}\\
 \cmidrule{1-4}
 DS-LSS & Direction Signs & Limit Speed Signs & 99.08\%\\
 LSS-LSS (Ours) & Limit Speed Signs & Limit Speed Signs & \textbf{99.39\%}\\
\bottomrule
\end{tabular}
}
\label{tab:task_change}
\end{table}

\paragraph{Study with task exchange on GTSRB}
To investigate the subtask-oriented channel selection, we further explored a task-exchange experiment.
We compared the original CATRO model on one target task (e.g., Direction Signs) with a task-exchange CATRO model, which is pruned (i.e., channel selection) on another task (e.g., Limited Speed Signs) and fine-tuned on the target task (Direction Signs). These two models are named by ``DS-DS'' and ``LSS-DS'', respectively.
We expected the model pruned and fine-tuned on the same task would perform better, and the results shown in Table~\ref{tab:task_change} confirms it.
Comparing “LSS-DS” and “DS-DS”,  we find the accuracy of Direct-Sign with task-exchanging channels drops $0.54\%$. Similarly, we can find the accuracy drop by comparing “LSS-LSS” and “DS-LSS”. The results demonstrate that CATRO is a task-oriented pruning approach and its selected/pruned channels are based on the specific subtask.

%% file: exp-ablation-alg2.tex
\begin{table*}[t]
\centering
\caption{Ablation study on the greedy architecture optimization (Algorithm~\ref{alg2}) on CIFAR-10.}
\resizebox{0.83\linewidth}{!}
{
\begin{tabular}{c c |c c c c}
\toprule
Network & Method & Predefined Pruning Ratio or Structure & FLOPs & Accuracy & Acc. Drop $\downarrow$ \\
\midrule \midrule
\multirow{4}{*}{ResNet-56} & SFP~\cite{he2018soft} & 40\% & 5.94E7 & 93.35\% & 0.24\% \\
& FPGM~\cite{he2019filter}  & 40\% &  5.94E7 & 93.49\% & 0.10\% \\
& CATRO-W/O Algorithm~\ref{alg2} & 40\% & 5.94E7 & 93.83\% & 0.08\% \\
\cmidrule(r){2-6}
& CATRO (Ours) & -- & 5.83E7 & 93.87\% & 0.03\% \\
\midrule
\multirow{7}{*}{ResNet-20}
& Taylor FO~\cite{molchanov2019importance} & Model-A (Table 4 in \cite{molchanov2019importance}) & 2.83E7 & 91.52\% & 0.48\%\\
& CATRO-W/O Algorithm~\ref{alg2} &  Model-A (Table 4 in \cite{molchanov2019importance}) & 2.83E7 & 92.30\% & 0.44\%\\
\cmidrule(r){2-6}
& SFP~\cite{he2018soft}  & 30\% & 2.43E7 & 90.83\% & 1.37\% \\
& FPGM~\cite{he2019filter} & 30\% & 2.43E7 & 91.09\% & 1.11\% \\
& CATRO-W/O Algorithm~\ref{alg2} & 30\% & 2.43E7 & 91.68\% & 1.06\% \\
\cmidrule(r){2-6}
& CATRO (Ours) & -- & 1.95E7 & 91.76\% & 0.98\% \\
\bottomrule
\end{tabular}
}
\label{tab:ablation}
\end{table*}

\paragraph{Investigations with fixed compression ratio}

To evaluate the effectiveness of the proposed trace ratio criterion in layer-wise pruning, we provided ablation studies with fixed compressions ratio (either uniform pruning or by given structures). This ablation study followed the settings of the baselines, and we replaced Algorithm~\ref{alg2} (searching the pruning ratios of each layer) in {\catro} by directly using the same pruning ratios as baselines. The results in the Table~\ref{tab:ablation} clearly demonstrated the effectiveness of the proposed criterion.
On one hand, under the same compression ratio strategy, {\catro} (even without the global architecture optimization) clearly outperformed the baselines on both accuracies and accuracy drops.
On the other hand, incorporating the architecture optimization further boosted the effects of {\catro} and led to a more compact but powerful network. For example, 0.04\% accuracy improvement with 1.1E6 FLOPs reduction for ResNet-56, and 0.08\% accuracy improvement with 4.8E6 FLOPs reduction for ResNet-20.

%% file: exp-ablation-layer-group.tex
\paragraph{Investigations on different layer groups}

\begin{table}[t]
\centering
\caption{Ablation study on the layer-wise pruning with trace-ratio criterion (Algorithm~\ref{alg1}) on CIFAR-10. $\checkmark$ and $L_1$ refers to using Algorithm~\ref{alg1} and $l_1$ pruner in that layer group, respectively.}
\resizebox{.9\columnwidth}{!}
{
	 \begin{tabular}{c  c c c | c }
	 \toprule
     \multirow{2}{*}{Network} & \multicolumn{3}{c|}{Layer Group} & \multirow{2}{*}{Acc. (Drop $\downarrow$)} \\
     & Shallow & Middle & Deep & \\
	 \midrule
	 \midrule
	 \multirow{6.5}{*}{ResNet-110}
	 & $L_1$  & $L_1$ & $L_1$ & 93.67\% (1.15\%) \\
	 & $\checkmark$  & $L_1$ & $L_1$ & 93.80\% (1.02\%) \\
	 & $\checkmark$  & $\checkmark$ & $L_1$ & 94.26\% (0.56\%) \\
	 & $L_1$  & $\checkmark$ & $\checkmark$ & 94.40\% (0.42\%) \\
	 & {$L_1$}  & {$L_1$} & {$\checkmark$} & {93.97\% (0.85\%)} \\
     \cmidrule(r){2-5}
     & \multicolumn{3}{c|}{{\catro} (Ours)}
     & 94.41\% (0.41\%) \\
	 \midrule
	 \multirow{6.5}{*}{ResNet-56}
	 & $L_1$  & $L_1$ & $L_1$ & 93.03\% (0.87\%) \\
	 & {$\checkmark$}  & {$L_1$} & {$L_1$} & {93.27\% (0.63\%)} \\
	 & $\checkmark$  & $\checkmark$ & $L_1$ & 93.52\% (0.38\%) \\
	 & $L_1$  & $\checkmark$ & $\checkmark$ & 93.73\% (0.17\%) \\
	 & {$L_1$}  & {$L_1$} & {$\checkmark$} & {93.50\% (0.40\%)} \\
     \cmidrule(r){2-5}
     & \multicolumn{3}{c|}{{\catro} (Ours)} 
     & 93.87\% (0.03\%) \\
	 \midrule
	 \multirow{6.5}{*}{ResNet-32}
	 & $L_1$  & $L_1$ & $L_1$ & 92.92\% (1.02\%)  \\
	 & {$\checkmark$}  & {$L_1$} & {$L_1$} & {92.94\% (1.00\%)} \\
	 & $\checkmark$  & $\checkmark$ & $L_1$ & 92.98\% (0.96\%) \\
	 & $L_1$  & $\checkmark$ & $\checkmark$ & 93.04\% (0.90\%) \\
	 & {$L_1$}  & {$L_1$} & {$\checkmark$} & {93.00\% (0.94\%)} \\
     \cmidrule(r){2-5}
     & \multicolumn{3}{c|}{{\catro} (Ours)} 
     & 93.17\% (0.77\%) \\
	 \midrule
	 \multirow{6.5}{*}{ResNet-20}
	 & $L_1$  & $L_1$ & $L_1$ & 91.19\% (1.55\%) \\
	 & {$\checkmark$}  & {$L_1$} & {$L_1$} & {91.30\% (1.44\%)} \\
	 & $\checkmark$  & $\checkmark$ & $L_1$ & 91.48\% (1.26\%) \\
     & $L_1$  & $\checkmark$ & $\checkmark$ & 91.64\% (1.10\%) \\
	 & {$L_1$}  & {$L_1$} & {$\checkmark$} & {91.54\% (1.20\%)} \\
	 \cmidrule(r){2-5}
     & \multicolumn{3}{c|}{{\catro} (Ours)}
     & 91.76\% (0.98\%) \\
	 \bottomrule
	 \end{tabular}
}
\label{tab:stages}
\end{table}

To further figure out to which extent the feature discriminations and the proposed criteria benefit channel pruning in {\catro}, we conducted additional ablation studies on the layer-wise pruning method in different parts of the network.
To be more specific, for each ResNet model, we grouped all layers into the Shallow-Group (the first one-third conv layers), the Middle-Group (the second one-third conv layers), and the Deep-Group (the last one-third conv layers).
With the same channel numbers obtained by Algorithm~\ref{alg2}, we replaced the layer-wise pruning (Algorithm~\ref{alg1}) in {\catro} by the common $l_1$ pruner in different layer groups.
As shown in Table~\ref{tab:stages}, we can find that removing trace criterion in either shallow or deep stages leads to accuracy drop, and removing trace criterion in the Deep-Group leads to more accuracy drop compared to the Shallow-Group.
Therefore, we can conclude that using discrimination in either shallow or deep stages both help and the proposed feature discrimination criteria has its superiority over common baselines.
This also lead to an open and exciting topic on better utilizing the discrimination or combining it with other pruning strategies, which we will explore in future work.

%% file: exp-discussion.tex
\begin{figure}[h]
    \centering
    \includegraphics[width=1.0\columnwidth]{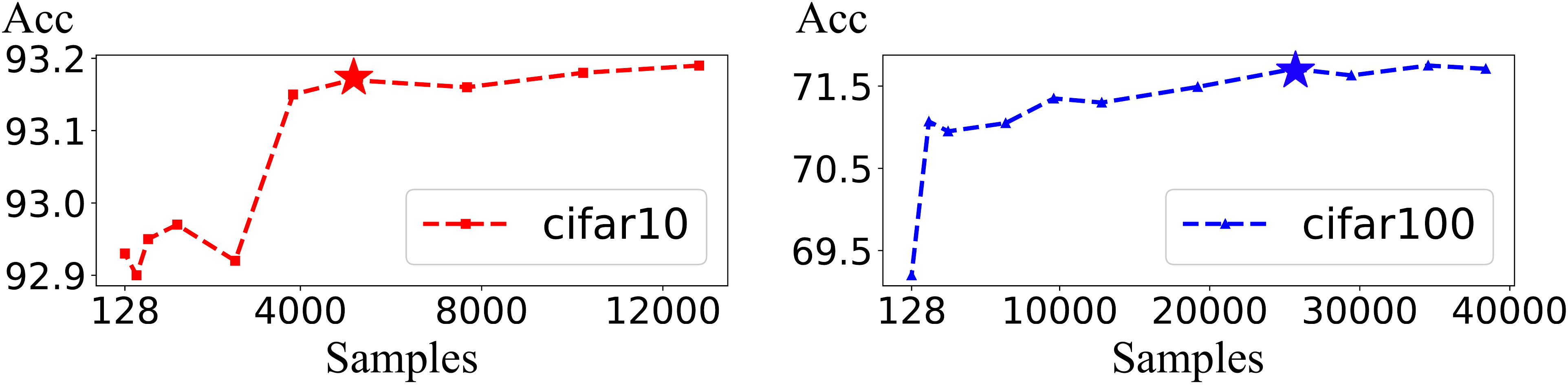}
    \caption{Comparisons of data sample numbers for ResNet-32 on the CIFAR datasets.
    Left: CIFAR-10 with a FLOPs drop of 56.1\%.
    Right: CIFAR-100 with a FLOPs drop of 43.8\%.
    }
    \label{fig:sample}
\end{figure}

\paragraph{Studies on number of data samples used in {\catro}}
The number of data samples used in the trace ratio calculation is a key factor.
In order to investigate the effects of the numbers of data samples, we studied different numbers of samples on CIFAR-10 and CIFAR-100. Results are shown in Fig.~\ref{fig:sample}, and we found that more data provided limited gains with enlarging samples. Furthermore, we found that CATRO needed fewer samples for each class with more classes, which may be due to the help of more negative samples from other classes. Since time complexity of {\catro} is independent to the number of training samples or batches, a reasonable number of data is a number that is small but represents the data distribution. In our experiments, we empirically randomly sampled 5,120/25,600 samples for CIFAR-10/CIFAR-100, which is marked with star in Fig.~\ref{fig:sample}.

\paragraph{Pruning time complexity of {\catro}}
The main difference in pruning time lies in the trace ratio computation steps.
The time complexity is linear to the number of sampled data which is usually much smaller than original training data.
We provided a detailed time analysis on different phases of the training schedule on a machine with Tesla P40 GPU and E5-2630 CPU.
We conducted the experiment for ResNet-110 on CIFAR-10 and used 5,120 out of the 50K samples.
The one-time forward propagation took about $22.17s$ (including feature map transfer from GPU to CPU), and the average trace optimization for one layer on CPU took about $0.82s$.
The time of Algorithm~\ref{alg2} is linear to the time of updating $\lambda$. Note that $\lambda$ in Algorithm~\ref{alg2} can be calculated in parallel in advance, Algorithm~\ref{alg2} only took about $1.0s$ to perform the greedy search on the CPU, while common training forward and backward propagation on GPU took about $63s$ per epoch on average.
Overall, the pruning time for CATRO was about $22.17 + 1.0 + 0.82\times 110 =113.37s$.
Therefore, CATRO needs much less time for the pruning step than other training-based methods, e.g., GBN~\cite{you2019gate}, which takes 10 epochs for each tock phase with $10\times 63s/epoch=630s$ and many rounds of tock phases.
As {\catro} has similar fine-tuning time to other pruning methods and much less pruning time, it is a much more efficient pruning method.

%% file: conclusion.tex
In this paper, we present {\catro}, a novel channel pruning method via class-aware trace ratio optimization.
By investigating and preserving the joint impact of channels, which is rarely considered in existing pruning methods, {\catro} can coherently determine the network architecture and select channels in each layer, both in an efficient non-training-based manner.
Both empirical studies and theoretical justifications have been presented to demonstrate the effectiveness of our proposed {\catro} and its superiority over other state-of-the-art channel pruning algorithms.

%% file: bio.tex
\begin{IEEEbiography}[{\includegraphics[width=1in,height=1.25in,clip,keepaspectratio]{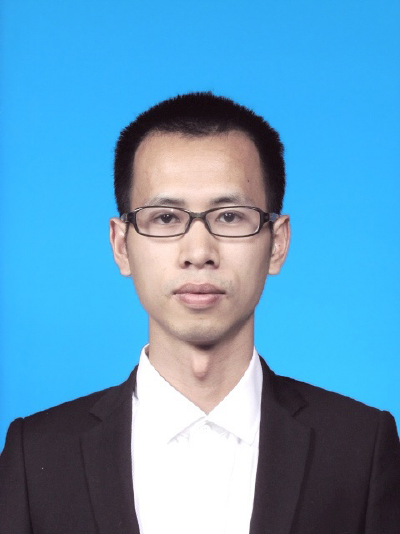}}]{Wenzheng Hu}
received the B.S. degree from the Department of Automation, Beihang University, Beijing, China, in 2013, and Ph.D. degrees in control science and engineering from Tsinghua University, Beijing, China, in 2019. He is
a Postdoc in the State Key Laboratory of Automotive Satefy and Energy, Tsinghua University, Beijing, China, and he is a researcher at Kuaishou Technology Co., Ltd., P.R.China. His research interests include machine learning, deep learning and intelligent transportation
systems.
\end{IEEEbiography}

\begin{IEEEbiography}[{\includegraphics[width=1in,height=1.25in,clip,keepaspectratio]{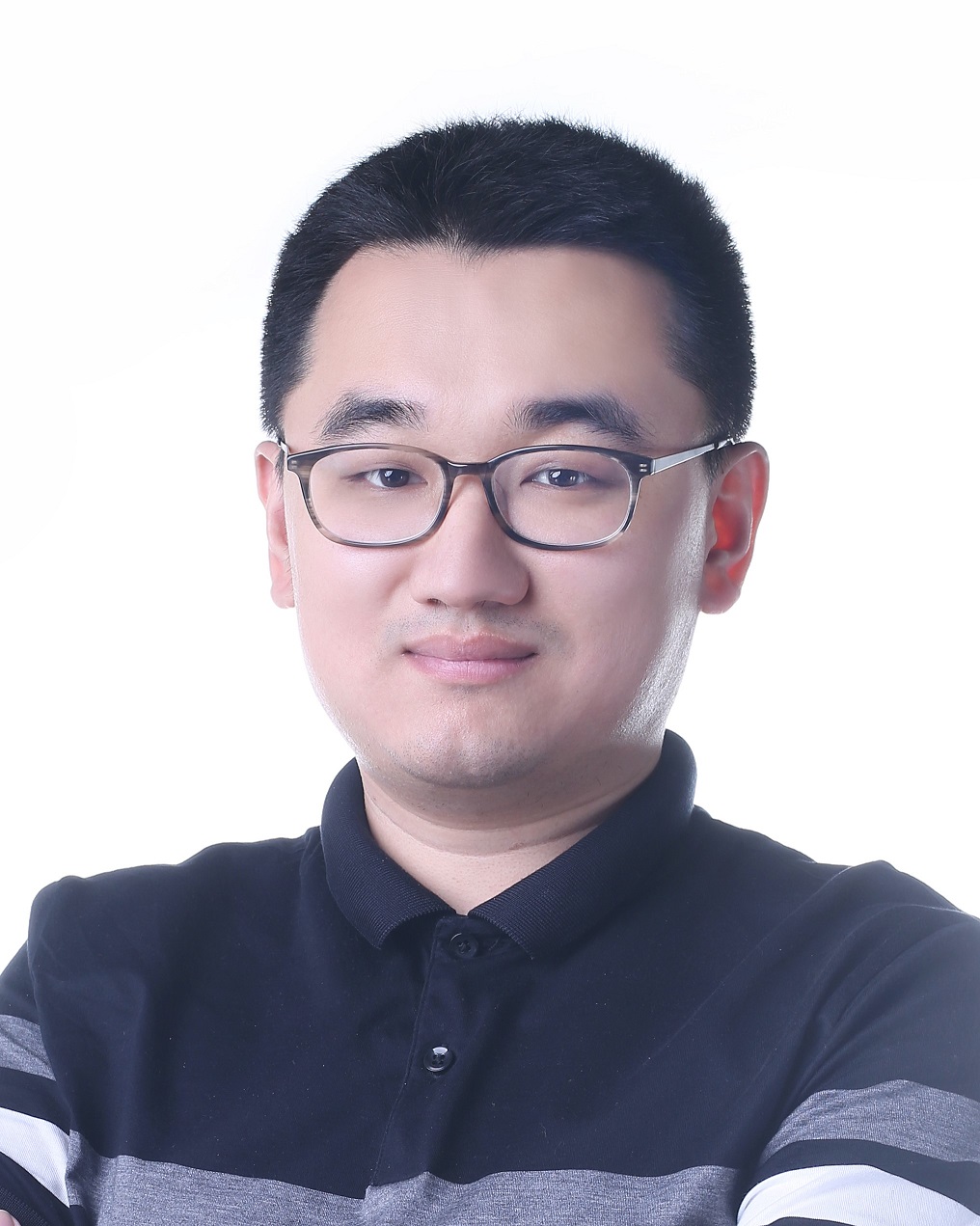}}]{Zhengping Che}
received the Ph.D. degree in Computer Science from the University of Southern California, Los Angeles, CA, USA, in 2018, and the B.Eng. degree in Computer Science from Tsinghua University, Beijing, China, in 2013.
He is now with AI Innovation Center, Midea Group.
His research interests lie in the areas of deep learning, computer vision, and time series analysis with applications to robot learning.
\end{IEEEbiography}

\begin{IEEEbiography}[{\includegraphics[width=1in,height=1.25in,clip,keepaspectratio]{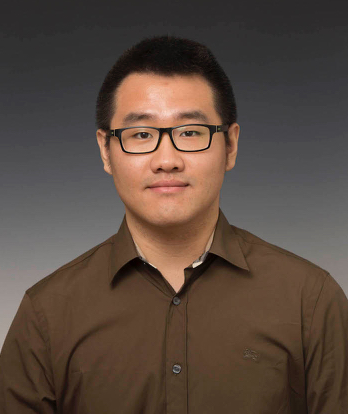}}]{Ning Liu}
received the Ph.D. degree in computer engineering from the Northeastern University, Boston, MA, USA, in 2019. He is a researcher at Midea Group. His current research interests lie in deep learning, deep model compression and acceleration, deep reinforcement learning, and edge computing.
\end{IEEEbiography}

\begin{IEEEbiography}[{\includegraphics[width=1in,height=1.25in,clip,keepaspectratio]{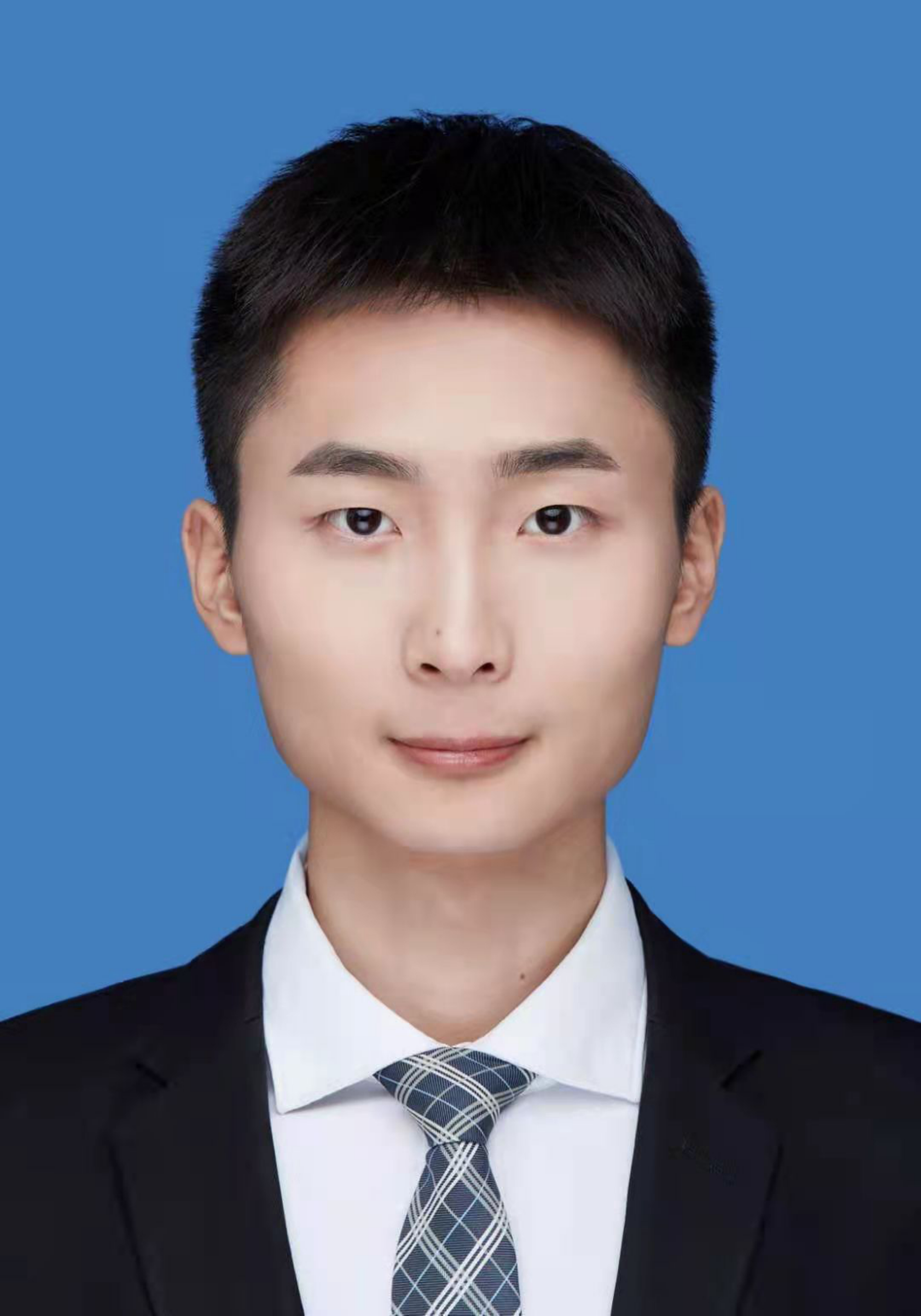}}]{Mingyang Li}
received the Master degree from in
Department of Automation, Tsinghua University, Beijing,
China, in 2020. He is currently an engineer in KE Holdings Inc., Beijing, China. His research interests include machine learning and computer vision.
\end{IEEEbiography}

\begin{IEEEbiography}[{\includegraphics[width=1in,height=1.25in,clip,keepaspectratio]{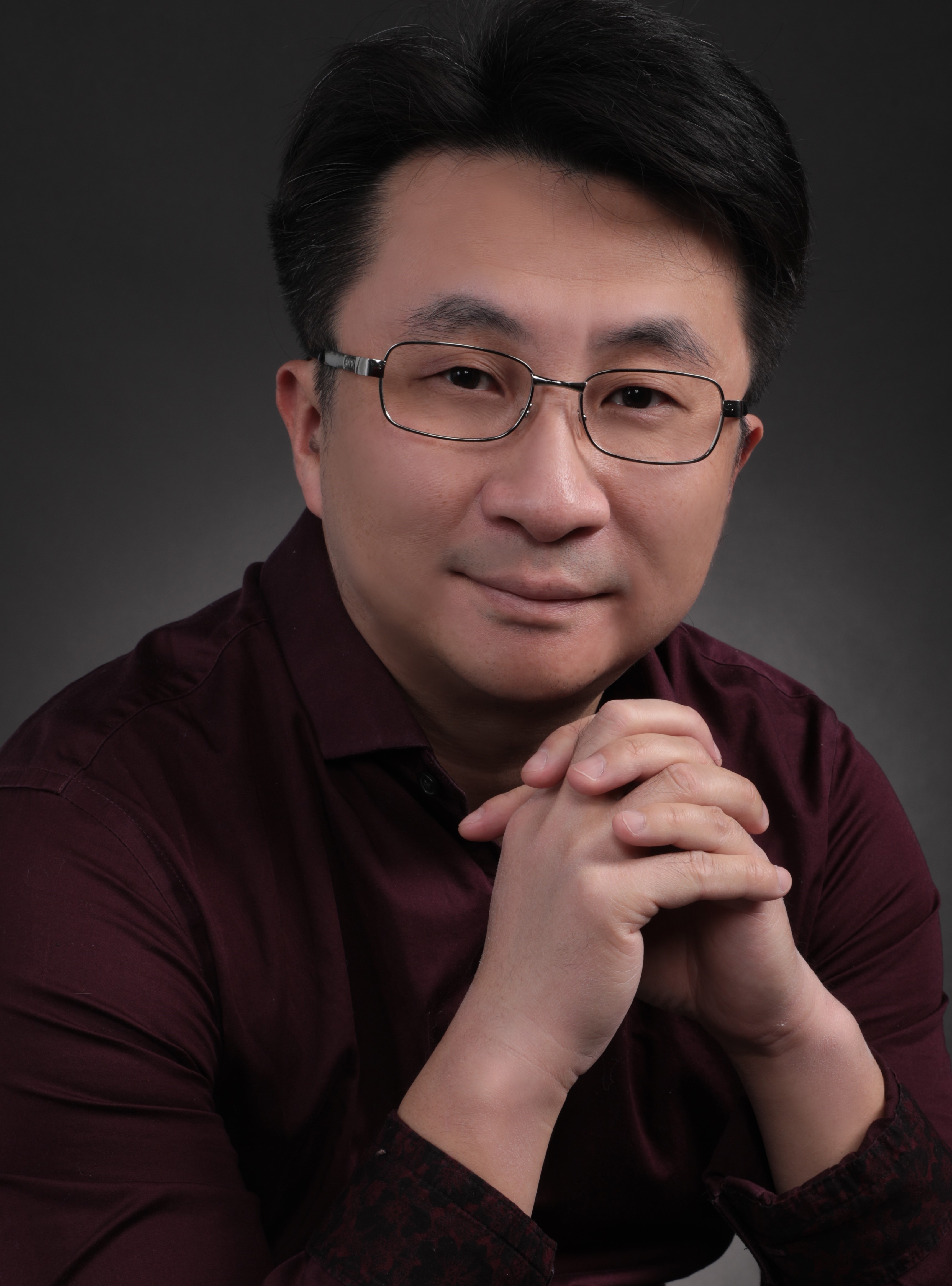}}]{Jian Tang}
(F’2019) received his Ph.D degree in Computer Science from Arizona State University in 2006. He is an IEEE Fellow and an ACM Distinguished Member. He is with Midea Group.
His research interests lie in the areas of AI, IoT, Wireless Networking and Mobile Computing. Dr. Tang has published over 180 papers in premier journals and conferences. He received an NSF CAREER award in 2009. He also received several best paper awards, including the 2019 William R. Bennett Prize and the 2019 TCBD (Technical Committee on Big Data) Best Journal Paper Award from IEEE Communications Society (ComSoc), the 2016 Best Vehicular Electronics Paper Award from IEEE Vehicular Technology Society (VTS), and Best Paper Awards from the 2014 IEEE International Conference on Communications (ICC) and the 2015 IEEE Global Communications Conference (Globecom) respectively. He has served as an editor for several IEEE journals, including IEEE Transactions on Big Data, IEEE Transactions on Mobile Computing, etc. In addition, he served as a TPC co-chair for a few international conferences, including the IEEE/ACM IWQoS'2019, MobiQuitous'2018, IEEE iThings'2015. etc.; as the TPC vice chair for the INFOCOM'2019; and as an area TPC chair for INFOCOM 2017-2018. He was also an IEEE VTS Distinguished Lecturer, and the Chair of the Communications Switching and Routing Committee of IEEE ComSoc.
\end{IEEEbiography}

\begin{IEEEbiography}[{\includegraphics[width=1in,height=1.25in,clip,keepaspectratio]{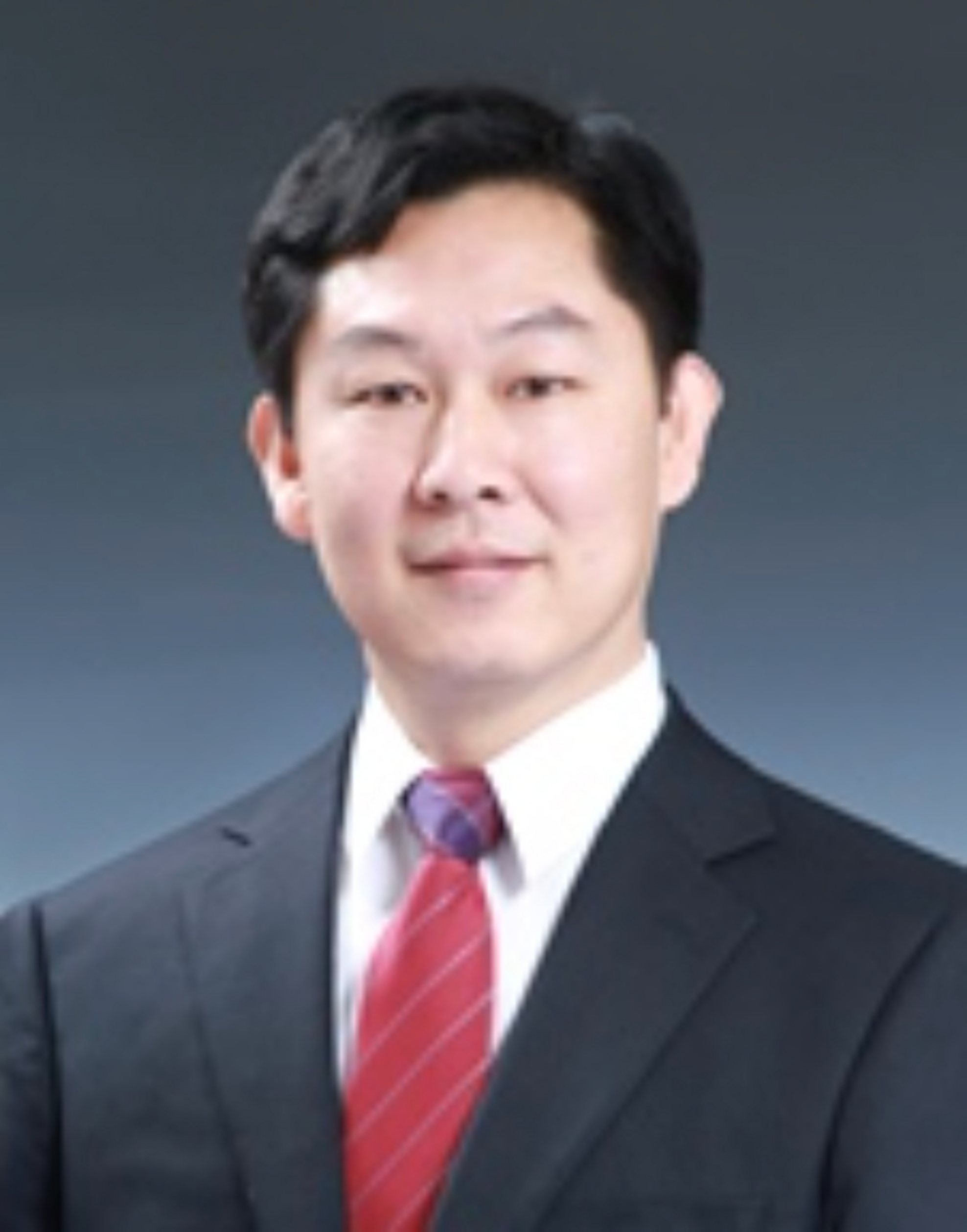}}]{Changshui Zhang}
(M’02–SM’15–F’18) received the B.S. degree in mathematics from Peking University, Beijing, China, in 1986, and the M.S. and Ph.D. degrees in control science and engineering from Tsinghua University, Beijing, China, in 1989 and 1992, respectively. In 1992, he joined the Department of Automation, Tsinghua University, where he is currently a Professor. He has authored more than 200 articles, and
he is a member of the Standing Council of the Chinese Association
of Artificial Intelligence. Dr. Zhang is currently an Associate Editor of the Pattern
Recognition Journal and the IEEE TPAMI, and he is also a Fellow member of IEEE. His current research interests include pattern recognition
and machine learning.
\end{IEEEbiography}

\begin{IEEEbiography}[{\includegraphics[width=1in,height=1.25in,clip,keepaspectratio]{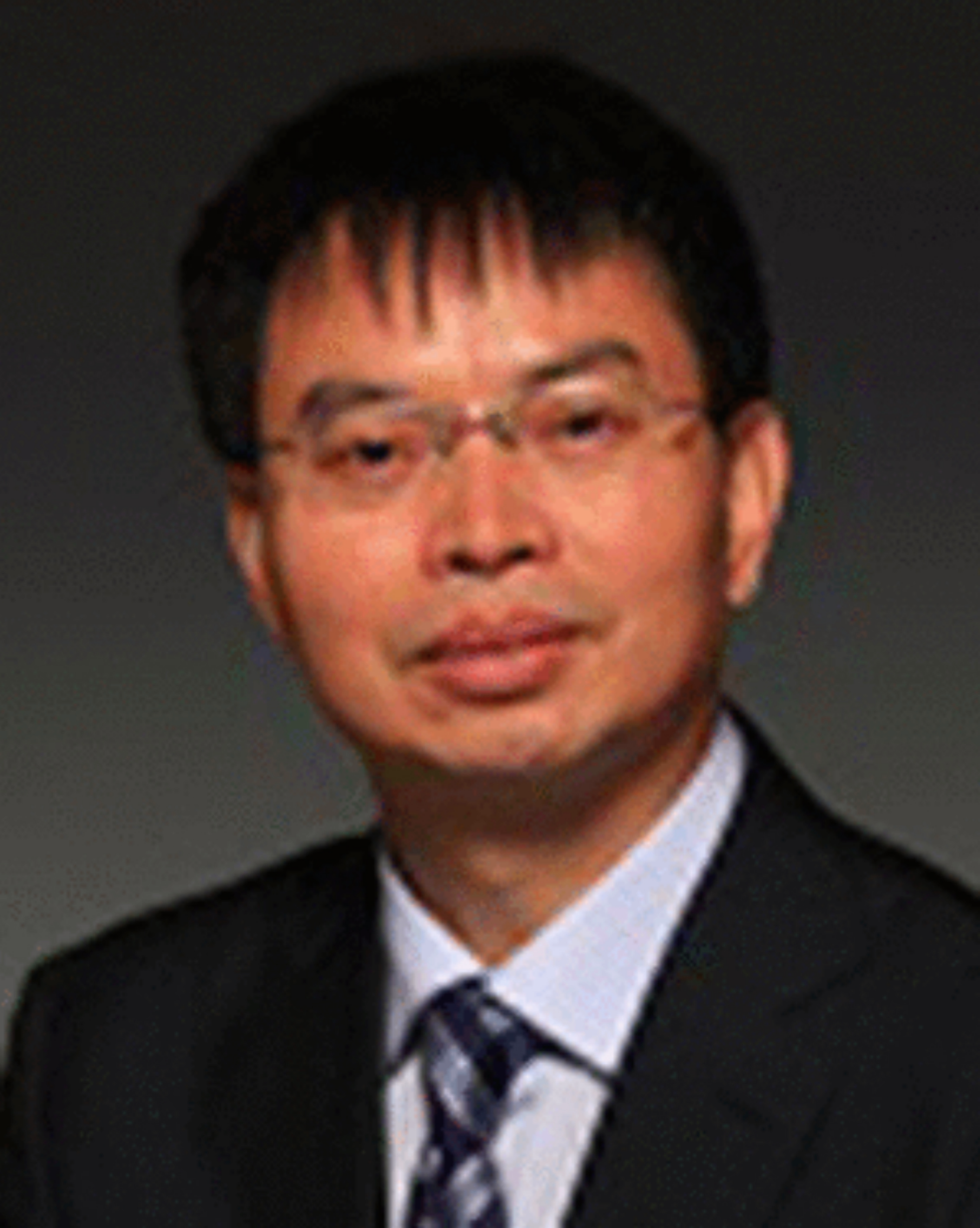}}]{Jianqiang Wang}
received the B.Tech. and M.S. degrees in automotive application engineering from the Jilin University of Technology, Changchun, China, in 1994 and 1997, respectively, and the Ph.D. degree in vehicle operation engineering from Jilin University, Changchun, in 2002. He is currently a Professor with the School of Vehicle and Mobility, Tsinghua University, Beijing, China. He has authored more than 40 journal papers and is currently a coholder of 30 patent applications. His research focuses on intelligent vehicles, driving assistance systems, and driver behavior. Dr. Wang was a recipient of the Best Paper Awards in IEEE Intelligent Vehicles (IV) Symposium 2014, IEEE IV 2017, 14th ITS Asia-Pacific Forum, the Distinguished Young Scientists of National Science Foundation China in 2016, and the New Century Excellent Talents in 2008.
\end{IEEEbiography}

%


